\documentclass[conference,letterpaper]{IEEEtran}
\usepackage{mathtools}
\usepackage{cite}
\usepackage{amssymb}
\usepackage{amsthm}
\usepackage{latexsym}
\usepackage{verbatim}
\usepackage{subfigure}
\usepackage{psfrag}
\usepackage{color}
\usepackage{enumitem}
\usepackage[utf8]{inputenc} 
\usepackage{amsfonts} 
\usepackage[T1]{fontenc}
\usepackage{url}              

\usepackage{hyperref}
\hypersetup{
	colorlinks=true,
	linkcolor=blue,
	filecolor=blue,
	citecolor = blue,      
	urlcolor=cyan,}

\usepackage{esdiff}
\usepackage{pgf,tikz,psfrag}
\usepackage{yfonts}

\def\mathlette#1#2{{\mathchoice{\mbox{#1$\displaystyle #2$}}%
		{\mbox{#1$\textstyle #2$}}%
		{\mbox{#1$\scriptstyle #2$}}%
		{\mbox{#1$\scriptscriptstyle #2$}}}}
\newcommand{\matr}[1]{\mathlette{\boldmath}{#1}}
\newcommand{\RR}{\mathbb{R}}

\newcommand{\NN}{\mathbb{N}}

\newtheorem{theorem}{Theorem}
\newtheorem{proposition}{Proposition}
\newtheorem{lemma}{Lemma}
\newtheorem{corollary}{Corollary}
\newtheorem{remark}{Remark}
\newtheorem{assumption}{Assumption}

\newtheorem{definition}{Definition}

{\begin{list}               
    {$\bullet$ \hfill}{
        \setlength{\leftmargin}{\parindent}
        \setlength{\parsep}{0.04\baselineskip}
        \setlength{\itemsep}{0.5\parsep}
        \setlength{\labelwidth}{\leftmargin}
        \setlength{\labelsep}{0em}}
    }
{\end{list}}

\providecommand{\sd}{K}

\providecommand{\inprod}[1]{\langle#1\rangle}

\providecommand{\bydef}{\stackrel{\Delta}{=}}

\providecommand{\cref}[1]{Chapter~\ref{chap:#1}}

\providecommand{\R}{\ensuremath{\mathbb{R}}}

\providecommand{\N}{\ensuremath{\mathbb{N}}}

\newcommand{\Op}[1]{\mathcal{O}(#1)}

\providecommand{\GS}[2]{\mathcal{GS}(#1 \mid #2)}

\newcommand{\av}{{\matr a}}
\newcommand{\bv}{{\matr b}}

\newcommand{\ev}{{\matr e}}
\newcommand{\fv}{{\matr f}}
\newcommand{\gv}{{\matr g}}
\newcommand{\hv}{{\matr h}}

\newcommand{\lv}{{\matr  l}}
\newcommand{\mv}{{\matr m}}

\newcommand{\rv}{{\matr r}}

\newcommand{\uv}{{\matr u}}

\newcommand{\vv}{{\matr v}}
\newcommand{\xv}{{\matr x}}
\newcommand{\yv}{{\matr y}}
\newcommand{\zv}{{\matr z}}

\newcommand{\Am}{{\matr A}}

\newcommand{\Gm}{{\matr G}}

\newcommand{\Id}{{\matr I}}

\newcommand{\Pm}{{\matr P}}

\newcommand{\Xm}{{\matr X}}

\newcommand{\Zm}{{\matr Z}}

\begin{document}
	\title{A Convergence Analysis of Approximate Message Passing with Non-Separable Functions \\
		and Applications to Multi-Class Classification} 
	


	\author{%
		\IEEEauthorblockN{Burak \c{C}akmak\IEEEauthorrefmark{1},
			Yue M. Lu\IEEEauthorrefmark{2},
			Manfred Opper\IEEEauthorrefmark{1}\IEEEauthorrefmark{3}\IEEEauthorrefmark{4}}
		\IEEEauthorblockA{\IEEEauthorrefmark{1}%
			Technical University of  Berlin,
			Berlin 10587, Germany,
			\{burak.cakmak, manfred.opper\}@tu-berlin.de}
		\IEEEauthorblockA{\IEEEauthorrefmark{2}%
			Harvard University, Cambridge, MA 02138, USA,
			yuelu@seas.harvard.edu}
		
		\IEEEauthorblockA{\IEEEauthorrefmark{3}%
			University of Birmingham, Birmingham
			B15 2TT, United Kingdom} 
		\IEEEauthorblockA{\IEEEauthorrefmark{4}%
			University of Potsdam, Potsdam 14469, Germany} 
	}

	\maketitle

	\begin{abstract}
		Motivated by the recent application of approximate message passing (AMP) to the analysis of convex optimizations in multi-class classifications [Loureiro, et. al., 2021], we present a convergence analysis of AMP dynamics with non-separable multivariate nonlinearities. As an application, we present a complete (and independent) analysis of the motivated convex optimization problem.
	\end{abstract}
	
	\section{Introduction}
	The analysis of the statistical properties of high-dimensional random convex optimization
	problems is a very active topic in information theory, statistics, machine learning \cite{Kabashima2009typical,Bayati11,
		vehkapera2016analysis,Christos18,Cedric23,thrampoulidis2020theoretical,thrampoulidis2015regularized,loureiro2021learning,celentano2022fundamental,
		cornacchia2022learning}. 
	The interest in studying such problems is motivated by the fact that they can serve as models
	for large-scale parameter estimation in statistical machine learning
	and communication theory, where the data are the source of randomness in the models.
	The  asymptotic limit considered in such models, 
	where both the number of data instances and the number of model parameters grow large (while their ratio is
	fixed), makes the theoretical analysis non-trivial.
	
	Research on this topic was dominated for a long time by non--rigorous approaches, such as the replica method, developed in the field of statistical physics \cite{Kabashima2009typical,Rangan_2012,vehkapera2016analysis}. In recent years, a  variety of models have been treated rigorously by methods of high--dimensional probability \cite{Bayati11,Christos18,celentano2022fundamental}.
	
	A promising, but somewhat indirect approach derives properties of the \textit{static} optimization problems by
	studying the \textit{dynamics} of a class of algorithms which are constructed to converge to their solutions. They are usually known as approximate message passing (AMP) style algorithms. 
	For high-dimensional problems the dynamics of individual nodes in the AMP algorithms 
	can often be asymptotically decoupled, leading to an effective equivalent stochastic dynamics for
	a single node. The high--dimensional dynamical problem is 
	replaced by the time evolution (usually termed ``state evolution'') of a finite number of deterministic ``order parameters'' which converge
	to the corresponding order parameters describing the static properties of the optimization
	problem.  The necessary convergence properties of AMP--style algorithms for cases
	where dynamical nodes consist of \textit{scalar} random variables 
	are often obtained in a relatively simple way from the 
	contraction properties of a one-dimensional nonlinear mapping which governs the dynamics of order parameters. In such cases, parameter values of the algorithm which separate regions of (local) 
	convergence from a divergent behavior can usually be shown to be related to the so-called
	\textit{de Almeida--Thouless} (AT) \cite{AT} stability criterion of the static problem \cite{Bolthausen,Opper16,ccakmak2022analysis,takahashi2022macroscopic}. 
	
	We study AMP-style algorithms with non-separable functions\cite{javanmard2013state, Berthier20,zhong2021approximate,gerbelot2021graph}
	(i.e., the so-called denoiser function is vector-valued and has non-separable nonlinearities over indices) and their application to 
	the analysis of convex optimization problems. These more complex models have become increasingly interesting, because of their practical  
	relevance ranging from multi-class classification in machine learning \cite{thrampoulidis2020theoretical,loureiro2021learning,cornacchia2022learning} to channel estimation in communication theory \cite{fengler,ccakmak2024inference}.
	
	The analysis of the asymptotic convergence properties in the high--dimensional
	limit is less straightforward. The order parameters in the state evolution are matrices for which contraction properties are not easily obtained.
	While recent works \cite{loureiro2021learning,cornacchia2022learning} show promising results in this direction, we conjecture that such results may not be complete. E.g., explicit (AT style) stability criteria have not been obtained so far.
	
	Our main contribution is to leverage the convexity properties of the state evolution mapping and present a complete (and novel) analysis of the convergence properties of the AMP dynamics involving non-separable functions. It reveals that a simplified approach relying solely on monotonicity 
	properties of the mapping, akin to the one employed by Loureiro et al. (2021) \cite{loureiro2021learning}, is insufficient to establish convergence (see Remark~\ref{remforkrazakala}).  Also, our analysis does not rely on several model assumptions made in \cite{loureiro2021learning,cornacchia2022learning}, such as the uniqueness of the fixed-point solution in the state evolution and the boundedness of the observation vectors.
	
	\subsubsection*{Organization} In Section~\ref{Mainresults}, we present a contraction mapping analysis of AMP dynamics with non-separable functions. Section~\ref{Sec_Application} applies this analysis to convex optimization. Section~\ref{Simulations} provides numerical illustrations of AT-type stability. Conclusions are in Section~\ref{Conclusion}. 
	Proofs of intermediate results are in the appendix.
	\vspace{-0.1cm}
	\subsection{Notations}\label{NotDef}
	We use the index set notation, $[N] \bydef \{1, \ldots, N\}$. The scaling parameters in the paper are $d$ and $n$ while their ratio $\alpha\bydef n/d$ is fixed as $d, n \to \infty$. We write $m\asymp d$ to imply that the ratio $m/d$ is fixed as $d,m\to \infty$. Throughout  the paper, $\sd, T\in \N$ are fixed (w.r.t. $d$). For $m \asymp d$, we use bold-faced lower case letters, e.g., $\av, \bv$, to denote $m\times K$ matrices, whose (normalized) inner product is defined as
	\begin{equation}\label{eq:nip}
		\inprod{\matr a,\matr b} \bydef \frac 1  m \matr a^\top\matr b\;.
	\end{equation}
	The $k$th column and the $i$th row of $\matr a\in \RR^{m\times K}$ are denoted by $\matr a_k\in \RR^{m\times 1}$ and ${a}^{[i]}\in \RR^{1\times K}$, respectively. We use calligraphic letters, e.g. $\mathcal A$, for $\sd\times \sd$ matrices and $\mathcal I$ stands for the $\sd\times \sd$ identity matrix. $\mathcal A> \matr 0$ (resp.  $\mathcal A \geq \matr 0$) implies that $\mathcal A$ is a positive definite (resp. positive semi-definite) matrix. 
	
	The multivariate Gaussian distribution (resp. density function of $\xv$) with mean $\matr \mu$ and covariance matrix $\matr \Sigma$ is denoted by $\mathcal N(\matr \mu, \matr\Sigma)$ (resp. $\textswab{g}(\xv\vert\matr \mu,\matr \Sigma)$). $A\sim \rm P$ indicates that a random variable or random vector (rv or RV) $A$ has a distribution~$\rm P$. We use the notation $\av \sim_{\text{i.i.d.}} A$ to indicate that the rows of the vector $\av$ are independent and identically distributed (i.i.d.) copies of the RV $A$.  We may also indicate this as $\av \sim_{\text{i.i.d.}} {\rm P}$ where ${\rm P}$ denotes the distribution of $A$.
	

	\subsubsection*{Concentrations with $\mathcal L^p$ norms} 
	Let $\kappa_d>0$ be a deterministic sequence indexed by $d$, e.g., $\kappa_d= d^{-\frac 1 2}$ or $\kappa_d=1$. We~write
	\[A= \mathcal{O}(\kappa_d) \]
	to imply that for each $p\in \NN$ there is a constant $C_p$ such that $\Vert A \Vert_{\mathcal L^p}\leq C_p\kappa_d$ with $\Vert A\Vert_{\mathcal L^p}\bydef (\mathbb E\vert A\vert^p)^\frac {1}{p}$.
	We say $\hat{\matr a}$ is a high-dimensional equivalent of $\matr a$, denoted by
	\begin{equation}\label{eq:equiv}
		\hat{\matr a}\simeq \matr a~~\text{if}~~ \Vert\hat{\matr a} -\matr a\Vert_{\texttt{F}}= \Op{1}.
	\end{equation}
	E.g., if $\hat{\matr a}\simeq \matr a$, then for any \emph{small} constant $c>0$ we have the almost sure convergence as $d\to \infty$
	\begin{equation}
		\frac{1}{d^{c}}\Vert \hat{\matr a}-\matr a \Vert_{\texttt{F}}=\Op{d^{-c}}   \overset{a.s.}{\rightarrow} 0. \label{ascon}
	\end{equation}
	Indeed, the definition of $\mathcal{O}(d^{-c})$ and Markov's inequality yield $\mathbb{P}(|\mathcal{O}(d^{-c})| \geq \epsilon) \leq \frac{C_p^p}{\epsilon^p} d^{-cp}$ for some $\epsilon > 0$ and $p \in \mathbb{N}$ and choosing $p>\frac 1 c$ leads from Borel-Cantelli's lemma to \eqref{ascon}. 
	
	\section{Analysis of the Contraction Mapping}\label{Mainresults}
	We consider the problem of reconstructing of an unknown matrix $\matr \omega_0\in \RR^{d\times K}$ from an observation matrix $\matr y\in \RR^{n\times K}$ which is generated according to a log-likelihood function
	\begin{equation}
		\ln p(\matr y\vert\matr\omega_0,\matr X)=\sum_{1\leq i\leq n} \ln p_{0}(y^{[i]}\vert\matr x_i^\top\matr \omega_0 )\label{model1}
	\end{equation}
	where $p_0(y\vert \theta)$ denotes \emph{generating} likelihood function and $\matr x_i^\top$ denotes the $i$th row of the weight matrix $\matr X\in \RR^{n\times d}$. As a concrete application, we will later consider the reconstruction of $\matr \omega_0$, denoted by $\matr \omega^\star$, by the following convex optimization
	\begin{equation}
		\matr \omega^\star=\underset{\matr \omega }{\arg\min}~\sum_{1\leq i\leq n} l(\matr x_i^\top\matr \omega; y^{[i]}) +
		\frac{\lambda_0}{2}\Vert \matr\omega \Vert_\texttt{F}^2 \label{opt0}
	\end{equation}
	with a convex loss function $l(\theta;y)$ (w.r.t. $\theta$) and a regularization constant $\lambda_0>0$. 
	The optimization \eqref{opt0}  can be interpreted as the maximum a posterior estimation of $\matr \omega_0$ 
	of a Bayesian inference approach
	with the \emph{assumed} prior $\matr \omega_0\sim_\text{i.i.d.} \mathcal{N}(\matr 0, \mathcal I/{\lambda_0})$ and the \emph{assumed} likelihood $p(y\vert \theta)\propto \exp(-l(\theta;y))$. 
	
	Similar to \cite{loureiro2021learning},  we consider a generic AMP algorithm to solve the optimization problem. 
	It is defined for $t\in[T]$ by the dynamics 
	\begin{subequations}\label{alg2}
		\begin{align}
			\matr \gamma^{(t)}&=\matr X{\matr \omega}^{(t)}- \matr f^{(t-1)}\label{alg21}\\
			\matr f^{(t)}&=f(\matr \gamma^{(t)}; \matr y)\label{alg22}\\
			\matr \omega^{(t+1)}&= \matr X^\top\matr f^{(t)}-\alpha{\matr \omega}^{(t)}{\mathcal Q}^{(t)}\;.\label{alg23}
		\end{align}
	\end{subequations}
	Here, $f(\matr \gamma;\matr y)$ is an $n\times K$ matrix with its $i$th row denoted by $f(\gamma^{[i]};y^{[i]})$ and $f(\cdot;\cdot): (\RR^{K}\times \RR^{K})\to \RR^{K}$ is an appropriately chosen vector-valued function such that the fixed-point of ${\matr \omega}^{(t)}$ coincides with the reconstruction of the unkown, i.e., $\matr \omega^\star$. 
	
	Our goal is to analyze the distances between AMP dynamical variables with a fixed number of memory-step differences (e.g., $\sqrt{d^{-1}}\Vert \matr \omega^{(t+\tau)}-\matr \omega^{(t)}\Vert_{\texttt{F}}$ for $\tau\geq1$) and show if these contract as $t$ increases in the high--dimensional limit $d\to\infty$.
	The dynamics is
	initialized from a randomly perturbed weight vector $\matr \omega_0$. To be specific, 
	we set $\matr f^{(0)}=\matr 0$ and 
	\begin{equation}
		\matr \omega^{(1)}=\matr r_0\mathcal B^{(1)}+\matr u^{(0)}\sqrt{\mathcal C^{(1,1)}}\label{initial}\;.
	\end{equation}
	Here $\matr u^{(0)}\sim_{\text{i.i.d}}\mathcal N(\matr 0,\mathcal I)$ is an \emph{arbitrary} (i.e., independent of everything else constructed so far) random matrix and $\mathcal B^{(1)},\mathcal C^{(1,1)}\in \RR^{K\times K}$ are some deterministic matrices with $\mathcal C^{(1,1)}\geq\matr 0$ and $\rv_0$ is from the QR decomposition
	\begin{equation}
		\matr w_0=\matr r_0\langle\matr r_0,\matr \omega_0 \rangle \quad \text {s.t.}\quad \langle  \matr r_0,\matr r_0 \rangle =\mathcal I. \label{qr}
	\end{equation}
	Note that this choice of initialization contains information about the true matrix $\matr \omega_0$ (through $\rv_0$). A fully random initialization corresponds to setting $\mathcal B^{(1)} = \matr 0$. 
	
	While the AMP dynamics \eqref{alg2} describes a nonlinear system of nodes coupled by the random matrix $\matr X$, one can show that (with an appropriate choice of ${\mathcal Q}^{(t)}$, see below) the dynamics of the nodes decouple as $d\to \infty$ and can be rewritten in terms of i.i.d. stochastic processes described by the state-evolution:
	\begin{definition}[State Evolution]\label{SE}
		Let $\{\Psi^{(t)}\in \R^{1\times K}\}_{t\in [T+1]}$ is a zero-mean Gaussian process with the two-time covariances $\mathcal C^{(t,s)}\bydef \mathbb E[(\Psi^{(t)})^\top\Psi^{(s)}]$ for all $t,s\in [T+1]$ recursively constructed as
		\begin{align}
			\mathcal B^{(t+1)}&=\alpha \mathbb E[G_0^\top f(\Gamma^{(t)};Y)]-\alpha\mathcal B^{(t)}\mathbb E[f'(\Gamma^{(t)};Y)]\label{state2}\\
			\mathcal C^{(t+1,s+1)}&=\alpha\mathbb E[f(\Gamma^{(t)};Y)^\top
			f(\Gamma^{(s)};Y)]\label{state1}
		\end{align}
		where $\mathcal C^{(1,t)}\bydef\mathcal C^{(1,1)}\delta_{t1}$ for all $t\geq 1$ with $\delta_{ts}$ denoting the Kronecker delta and we have introduced random vectors $\Gamma^{(t)}\bydef G_0\mathcal B^{(t)}+ \Psi^{(t)}$ with  $\Psi^{(t)}$ being independent of  
		\begin{equation}
			(Y,G_0)\sim p_0(Y\vert G_0\sqrt{\mathcal C_0})\textswab{g}(G_0\vert 0,\mathcal I)~~\text{with}~~\mathcal C_0\bydef \langle\matr \omega_0,\matr \omega_0 \rangle\;.\label{YG0}
		\end{equation}
		Finally, we have set $
		\mathcal Q^{(t)}\bydef\mathbb E[f'(\Gamma^{(t)};Y)]$ where $f'(\gamma;\cdot)$ denotes  the $K\times K$ Jacobian of $f(\gamma;\cdot)$ w.r.t. $\gamma$, with the entries 
		$[f'(\gamma;\cdot)]_{kk'}=\frac{\partial [f(\gamma;\cdot)]_{k'}}{\partial \gamma_k}$ for any $k,k\in[K]$.
	\end{definition}
	\begin{proposition}[Decoupling Principle]\label{Th1} 
		Let $\matr X\sim_{\text{i.i.d.}}\mathcal N(\matr 0,\matr I/d)$. Let $f(\gamma;y)$ be differentiable and Lipschitz continuous w.r.t $\gamma$ and $f(0;Y)=\Op{1}$ where $Y$ as in \eqref{YG0}.
		Define $\matr g_{0}\bydef \matr X\matr r_{0}$. Then,  we have for any $t\in[T]$ 
		\begin{align}
			\matr \gamma^{(t)}&\simeq  \matr g_0\mathcal B^{(t)}+\tilde{\matr \psi}^{(t)}\label{Gamma11}\\
			\matr \omega^{(t+1)}&\simeq \matr r_0\mathcal B^{(t+1)}+\matr \psi^{(t+1)}\label{Omega11},
		\end{align}
		where the two sequences $\{\tilde{\matr \psi}^{(t)}\}_{t\in[T]}$ and $\{{\matr \psi}^{(t+1)}\}_{t\in[T]}$ are independent with $\tilde{\matr \psi}^{(t)}\sim
		_{\text{i.i.d.}}\Psi^{(t)}$ and $\matr \psi^{(t)}\sim_{\text{i.i.d.}}\Psi^{(t)}$ with the stochastic process  $\Psi^{(t)}$ as in Definition~\ref{SE}.
	\end{proposition}
	\begin{proof}
		See Appendix~\ref{proof_Th1} for a sketch of the proof.  
	\end{proof}
	Note that the above result is non-asymptotic
	and we recall \eqref{eq:equiv} and \eqref{ascon} for the definition and its asymptotic implication  (as $d\to \infty$)  of the high dimensional equivalence notation $\hat\av \simeq \av$.

	With the decoupling principle, the contraction mapping problem essentially boils down to analyzing how the \emph{two-time} (i.e., $(t,s)$) state-evolution (see ~\eqref{state2} and \eqref{state1}) converge to their fixed points as the number of iterations increases. For this purpose, we assume that such fixed points exist. 
	\begin{assumption}\label{ass2}
		There exist the  matrices $\mathcal C^\star$ and $\mathcal B^\star$ satisfying
		\begin{align}
			\mathcal B^\star&=\alpha\mathbb E[G_0^\top f(\Gamma^\star;Y)]-\alpha \mathcal B^\star\mathbb E[f'(\Gamma^\star;Y)]\\
			\mathcal C^\star&=\alpha \mathbb E[f(\Gamma^\star;Y)^\top f(\Gamma^\star;Y)]\;,\label{sefixed}
		\end{align} 
		where $\Gamma^\star\bydef G_{0}\mathcal B^\star+G\sqrt{\mathcal C^\star}$
		with $G\sim \mathcal N(\matr 0,\mathcal I)$ and $(Y,G_0)$ (see \eqref{YG0}) being independent. 
	\end{assumption}
	
	\begin{theorem}\label{Th2}
		Let $\matr X\sim_{\text{i.i.d.}}\mathcal N(\matr 0;\matr I/d)$. Let $f(\gamma;y)$ be two-times differentiable and Lipschitz continuous w.r.t $\gamma$ and $f(0;Y)=\Op{1}$  where $Y$ as in \eqref{YG0}.  Suppose Assumption~\ref{ass2} holds. Chose $\mathcal B^{(1)}=\mathcal B^\star$ and $\mathcal C^{(1,1)}=\mathcal C^\star$. Then,  for all $\tau\geq 1$
		\begin{align}
			\frac{\Vert \matr \gamma^{(t+\tau)}-\matr \gamma^{(t)}\Vert_\texttt{F}}{\sqrt n}&<C\rho_{\rm AT}^{\frac t 2}+ \mathcal{O}(d^{-\frac 1 2})\label{good1}\\
			\frac{\Vert \matr \omega^{(t+\tau)}-\matr \omega^{(t)}\Vert_\texttt{F}}{\sqrt d}&<C \rho_{\rm AT}^{\frac t 2}+\mathcal{O}(d^{-\frac 1 2}),\label{good2}
		\end{align}
		where $C>0$ denotes a fixed constant (independent of $t$) and 
		\begin{equation}
			\rho_{\rm AT}\bydef\rho(\alpha(\mathbb E[f'(\Gamma^\star;Y)\otimes f'(\Gamma^\star;Y)])\;,\label{AT}
		\end{equation}
		with $\rho(\cdot)$ denoting the spectral radius of the matrix in the argument and $\otimes$ denoting the Kronecker product.
	\end{theorem}
	Here, we note from the property \eqref{ascon} that e.g. the result \eqref{good1} implies almost surely 
	\[\lim_{d\to \infty}\frac{\Vert \matr \omega^{(t+\tau)}-\matr \omega^{(t)}\Vert_\texttt{F}}{\sqrt d}<C\rho_{\rm AT}^{\frac t 2} \;.
	\]
	Furthermore, we use the notation $\rho_{\rm AT}$ to associate AT stability. Indeed, the AMP dynamics \eqref{alg2} is not stable in the region $\rho_{\rm AT}\geq 1$, specifically, when $\rho_{\rm AT}\geq 1$  we have $
	\lim_{t\to \infty}\Vert \mathcal C^\star-\mathcal C^{(t,t+1)}\Vert_{\rm F}\neq 0$ (see Appendix ~\ref{proofremat}). 
	\subsection{The Proof of Theorem~\ref{Th2}}\label{PTh2}
	For $\mathcal B^{(1)}=\mathcal B^\star$ and $\mathcal C^{(1,1)}=\mathcal C^\star$,  from Definition~\ref{SE} it follows  inductively (over iterations steps)   that
	\begin{align}
		\mathcal B^{(t)}=\mathcal B^\star\text{~~and~~} \mathcal C^{(t,t)}=\mathcal C^\star\qquad \forall t. \label{vnice}
	\end{align}
	Moreover, since $\mathcal C^{(1,t)}=\matr 0$ for all $t>1$, it is also easy to verify from \eqref{vnice} that
	\begin{equation}
		\mathcal C^{(t,t+\tau)}=\mathcal C^{(t,t+1)}=\mathcal C^{(t+1,t)}\quad \forall \tau\geq 1. \label{vnice2}
	\end{equation}
	Then, we have
	\begin{align}
		\frac{\Vert \matr \omega^{(t+\tau)}-\matr \omega^{(t)}\Vert_\texttt{F}^2}{ d}&\overset{(a)}{=}
		\frac{\Vert \matr \psi^{(t+\tau)}-\matr \psi^{(t)}\Vert_\texttt{F}^2}{d}+ \mathcal{O}(d^{-\frac 1 2})\\
		&\overset{(b)}{=}2 \big\Vert {\Delta^{(t)}} \big\Vert_{\texttt{F}} +\mathcal{O}(d^{-\frac 1 2})\;,\label{omegadif}
	\end{align}
	where we have defined the single-time-step deviation operator
	\begin{equation}
		\Delta^{(t)}\bydef \mathcal C^\star-\mathcal C^{(t,t+1)}\label{ttd}.
	\end{equation}
	Here, steps (a), and (b) use Proposition~\ref{Th1} and Lemma~\ref{eq:ip_concentration} in Appendix~\ref{preliminariesop}, respectively.
	Then, from \eqref{eq:op_sqrt} we have \begin{equation}
		\frac{\Vert \matr \omega^{(t+\tau)}-\matr \omega^{(t)}\Vert_\texttt{F}}{\sqrt d}= \sqrt{2} \big\Vert {\Delta^{(t)}} \big\Vert_{\texttt{F}}^{\frac 1 2} +\mathcal{O}(d^{-\frac 1 2})\;.
	\end{equation}
	Similarly, we have 
	\begin{equation}
		\frac{\Vert \matr \gamma^{(t+\tau)}-\matr \gamma^{(t)}\Vert_\texttt{F}}{\sqrt n}= \sqrt{2} \big\Vert {\Delta^{(t)}} \big\Vert_{\texttt{F}}^{\frac 1 2} +\mathcal{O}(d^{-\frac 1 2})\;.
	\end{equation}
	Hence, we only need to verify that  $\Vert\Delta^{(t)}\Vert_\texttt{F}<C\rho_{\rm AT}^{t}$ for a constant $C$ independent of $t$. To this end, 
	inspired by the mapping $\psi$ in \cite[Lemma~2.2]{Bolthausen}, we introduce the mapping for $\matr 0\leq \mathcal X\leq \mathcal C^\star$
	\begin{align}
		\mathcal T(\mathcal X)\bydef\alpha&\mathbb E[f(G_0\mathcal B^\star+G\sqrt{\mathcal X}+ G'\sqrt{\mathcal C^\star-\mathcal X};Y)^\top \nonumber \\&\times f(G_0\mathcal B^\star+G\sqrt{\mathcal X}+ G''\sqrt{\mathcal C^\star-\mathcal X};Y)]\; \label{mappin} 
	\end{align}
	where the $K$.dim random vectors $\{G,G',G''\}$ are all mutually independent and distributed as $\mathcal N(\matr 0,\mathcal I)$, and they are independent of $(Y,G_0)$. E.g. notice that if $\matr 0\leq \mathcal C^{(t-1,t)}$ we have
	\begin{align}
		\Delta^{(t)}&=  \mathcal T(\mathcal C^{\star})-\mathcal T(\mathcal C^{(t-1,t)}).
	\end{align}
	\begin{lemma}\label{TContradiction}
		Let $f(\gamma;y)$ be two-times differentiable and Lipschitz continuous w.r.t. $\gamma$.  Suppose Assumption~\ref{ass2} holds. Then, we have for $\matr 0\leq \mathcal X\leq \mathcal C^\star$ 
		\begin{equation}
			\matr 0\overset{(a)}{\leq}\mathcal T(\mathcal X)-\mathcal T(\mathcal Y)\overset{(b)}{\leq} \alpha \mathbb E[f'(\Gamma^\star;Y)^\top (\mathcal X-\mathcal Y) f'(\Gamma^\star;Y)]\label{result1}.
		\end{equation}
		\begin{proof}
			See Appendix~\ref{dTContradiction}
		\end{proof}
	\end{lemma}
	
	\begin{remark}\label{remforkrazakala}
		It is not enough to employ the monotonicity, inequality (a) in \eqref{result1} to show convergence to  $\mathcal C^\star$ while (a)
		yields
		\begin{equation}
			\matr 0 \leq \mathcal C^{(t,t+1)}\leq \mathcal C^{(t+1,t+2)}\leq\mathcal C^\star\quad \forall   t
		\end{equation}
		which implies convergence $\mathcal C^{(t,t+1)}\to \hat{\mathcal C}^\star=\mathcal T(\hat{\mathcal C}^\star)$ with $\hat{\mathcal C}^\star\leq \mathcal C^\star$. Even assuming that $\mathcal C^\star $ has a unique solution of the fixed-point equation in \eqref{sefixed} does not imply that $\hat{\mathcal C}^\star= \mathcal C^\star$. A similar argument was misinterpreted in the proof of \cite{loureiro2021learning}.   
	\end{remark}
	\begin{definition}\label{vectorization}
		The vectorization of a $K\times K$ matrix $\mathcal X$, denoted $\overline{\mathcal X}$, is a $K^2\times 1$ column vector by stacking the column vectors of $\mathcal X=[\matr x_1, \matr x_2,\cdots,\matr x_K]$ below one another as 
		$\overline{\mathcal X}=[\matr x_1^\top,\matr x_2^\top,\cdots, \matr x_K^\top]^\top$.  Moreover, we say $\overline{\mathcal Y}\leq \overline{\mathcal X} \iff \mathcal Y\leq \mathcal X$, i.e. $(\overline{\mathcal X}-\overline{\mathcal Y})^\top(\matr u\otimes \matr u)\geq 0 $ for all $\matr u\in \RR^{K\times 1}$.
	\end{definition}

	For the arbitrary random elements $\mathcal X_t\sim f'(\Gamma^\star;Y)$ independent for each $t$ we write from Lemma~\ref{TContradiction} that
	\begin{align}
		\Delta^{(t+1)}&\leq\alpha \mathbb E[\mathcal X_t^\top  \Delta^{(t)}\mathcal X_t]\\
		&\leq\alpha^2 \mathbb E[(\mathcal X_{t-1} \mathcal X_{t})^\top \Delta^{(t-1)}(\mathcal X_{t-1}\mathcal X_t)]\\
		&~~\vdots \nonumber \\
		&\leq\alpha^t \mathbb E[(\mathcal X_{1}\mathcal X_{2}\cdots\mathcal X_{t})^\top \mathcal C^\star (\mathcal X_{1}\mathcal X_{2}\cdots \mathcal X_t)]
	\end{align}
	with noting that $\Delta^{(1)}=\mathcal C^\star$. Hence, in terms of the vectorization notation we have 
	\begin{align}
		\overline{\Delta^{(t+1)}}&\leq \alpha^t \mathbb E[\mathcal X_{1}\mathcal X_{2}\cdots\mathcal X_{t} \otimes  \mathcal X_{1}\mathcal X_{2}\cdots \mathcal X_t]^\top~\overline{\mathcal C^\star}\\
		&= \alpha^t \mathbb E[(\mathcal X_{1}\otimes \mathcal X_1)(\mathcal X_{2}\otimes\mathcal X_2)\cdots(\mathcal X_{t} \otimes\mathcal X_t)]^\top~\overline{\mathcal C^\star}\nonumber\\
		&=\left(\alpha \mathbb E[\mathcal X_1\otimes \mathcal X_1]^\top\right)^t~\overline{\mathcal C^\star}\label{key_result}
	\end{align}
	with noting that $\mathcal X_{t}$ and $\mathcal X_{t'}$ are independent for all $t\neq t'$.
	
	From \cite[Theorem~4.3.1]{horn2012matrix} we note that
	\[{\mathcal Y}\leq {\mathcal X} \implies \Vert \mathcal Y\Vert_{\texttt{F}}\leq   \Vert \mathcal X\Vert_{\texttt{F}}\;.\]
	Note also that $\Vert \mathcal X\Vert_{\texttt{F}}=\Vert \overline{\mathcal X}\Vert$. Then, from \eqref{key_result} we have
	\begin{align}
		\Vert {\Delta^{(t+1)}}\Vert_{\texttt{F}}&\leq \big\Vert \underbrace{\left(\alpha\mathbb E[f'(\Gamma^\star;Y)\otimes f'(\Gamma^\star;Y)]^\top\right)^{t}}_{\mathcal U\mathcal D^t\mathcal U^{-1}}\overline {\mathcal C^\star}\big\Vert\\
		&\leq \Vert\mathcal U\Vert_2\Vert\mathcal U^{-1}\overline {\mathcal C^\star} \Vert\rho_{\rm AT}^{t}
	\end{align}
	where the later inequality uses the eigenvalue decomposition $\mathcal U\mathcal D^t\mathcal U^{-1}$ with $\mathcal D$ being 
	diagonal. This completes the proof. 
	\section{Application to Analyzing Convex Optimizations}\label{Sec_Application}
	As an application of Theorem~\ref{Th2}, we analyze the convex optimization 
	\begin{equation}
		\matr \omega^\star\bydef  \underset{\matr \omega }{\arg\min}~\sum_{1\leq i\leq n} l(\matr x_i^\top\matr \omega; y^{[i]}) +
		\frac{\lambda_0}{2}\Vert \matr\omega \Vert_\texttt{F}^2 \label{opt}
	\end{equation}
	for a convex loss function $l(\theta;y)$ (w.r.t. $\theta$) and a regularization constant $\lambda_0>0$. Here, the data matrix $\matr y$ is assumed to be generated according to the log-likelihood \eqref{model1}. In particular, we may consider the so-called \emph{cross-entropy loss}--commonly used in multi-class classifications---
	\begin{equation}
		l(\theta;y)=-\sum_{k\leq K}y_{k}\ln\frac{e^{\theta_k}}{\sum_{k'} e^{\theta_{k'}}}\label{cel}\;,
	\end{equation}
	where $y$ is often defined as \emph{one-hot encoded vector}, i.e. $y=e_k$ for some $k\in[K]$ with $e_k$ denoting a $K$-dim. unit vector.
	
	Given an appropriate $K\times K$ deterministic matrix $\mathcal V^\star>\matr 0$ (see Assumption~\ref{ass3} below), we introduce the proximal operator
	\begin{equation}
		m(\gamma;y)\bydef  \underset{\theta}{\arg\min}\left( l(\theta;y)+\frac{1}{2}(\gamma-\theta)\mathcal V^\star(\gamma-\theta)^\top\right).\label{proximal}
	\end{equation}
	We set the non-linear function $f$ of the AMP dynamics \eqref{alg2} as
	\begin{equation}
		f(\gamma;y)= m(\gamma;y)-\gamma. \label{function}
	\end{equation}
	Here, it is worth to noting the relations
	\begin{align}
		f(\gamma;y)&=-l'(m(\gamma;y);y)(\mathcal V^\star)^{-1},
		\label{flrelation}\\
		f'(\gamma;y)&=-l''(m(\gamma;y);y)(\mathcal V^\star+l''(m(\gamma;y);y))^{-1}\;,
	\end{align}
	where $l'(\theta;y)$ and $l''(\theta;y)$ denote the gradient and Hessian  (w.r.t. $\theta$) of the loss function, respectively.

	\begin{assumption}\label{ass3}
		Let $\lambda_0>0$ and the function $f(\gamma;y)$ be as in \eqref{function}. There exist the $K\times K$ order matrices $\mathcal C^\star$, $\mathcal B^\star$, and $\mathcal V^\star>\matr 0$ satisfying the system of equations
		\begin{subequations}
			\label{fixed}	
			\begin{align}
				\mathcal V^\star&=\lambda_0\mathcal I -\alpha\mathbb E[f'(\Gamma^\star;Y)]\mathcal V^\star\label{fixed1}\\
				\mathcal B^\star&=\frac {\alpha}{\lambda_0}\mathbb E[G_0^\top f(\Gamma^\star;Y)]\mathcal V^\star\label{fixed2}\\
				\mathcal C^\star&=\alpha \mathbb E[f(\Gamma^\star;Y)^\top f(\Gamma^\star;Y)]\label{fixed3}\;,
			\end{align} 
		\end{subequations}
		where we define the random vector $\Gamma^\star=G_0\mathcal B^\star +G\sqrt{\mathcal C^\star}$ with $G\sim\mathcal N(\matr 0,\mathcal I)$ being independent of $(Y,G_0)$ (see \eqref{YG0}).
	\end{assumption}
	
	Note that Assumption~\ref{ass3} coincides with Assumption \ref{ass2}, i.e. existence of the fixed point of the state-evolution. 
	
	\begin{remark}\label{fixed-point}
		Let $f(\gamma;y)$ be as in \eqref{function} and $\mathcal V^\star$ as in Assumption~\ref{ass3}.  The fixed-point of $\matr\omega^{(t)}$ in the AMP dynamics \eqref{alg2} coincides with $\matr\omega^{\star}$ in \eqref{opt}.
		\begin{proof}
			See Equation \ref{goodbound}.
		\end{proof}
	\end{remark}

	\begin{assumption}\label{ass4}
		Let the loss function $l(\theta;y)$ be three times differentiable (w.r.t. to $\theta$) and convex. Let  $\Vert l''(\theta;y)\Vert_2$ be bounded for any $(\theta,y)$. Furthermore, let $l'(\matr 0;Y)=\Op{1}$ where the random variable $Y$ as in Assumption~\ref{ass3}. 
	\end{assumption}

	When $Y=\Op{1}$, the cross entropy loss function \eqref{cel} fulfills the conditions specified in Assumption~\ref{ass4}.  While in multi-class classification $Y$ is often defined as a one-hot coded vector such that $\Vert Y\Vert=1$, we note that the family of rvs $\Op{1}$ includes a wide range of distributions characterized 
	by "heavy" exponential tails, e.g. $A^D=\Op{1}$ for a sub-Gaussian rv $A$ and for any large (constant)~$D$. Hence, our analysis can be applied to a broader range of empirical risk minimization applications. 
	
	\begin{proposition}[The AT stability]\label{pat}
		Suppose Assumption~\ref{ass3} holds. Let the constant $\rho_{\rm AT}$ be as in \eqref{AT} with the function $f(\gamma;y)$ as in~\eqref{function}. Let the loss function $l(\theta;y)$ be two-times differentiable (w.r.t. to $\theta$) and convex. Then, $\rho_{\rm AT}<1$.
		\begin{proof}
			See Appendix~\ref{proof_AT}. 
		\end{proof}
	\end{proposition}
	
	\begin{theorem}\label{Th3}
		Let $\matr \omega ^\star $ be as in \eqref{opt}. Let $\matr X\sim_{\text{i.i.d.}}\mathcal N(\matr 0,\matr I/d)$. Suppose Assumptions~\ref{ass3} and \ref{ass4} hold.  Then, for any fixed (and large) $t\in \NN$ (independent of $d$) we have
		\begin{equation}
			\frac{\Vert \matr\omega^\star-(\matr r_0\mathcal B^\star+\matr u\sqrt{\mathcal C^\star}) \Vert_{{\rm F}}}{\sqrt d}<C\rho_{\rm AT}^{\frac t 2}+\Op{d^{-\frac 1 2}}\;,
		\end{equation}
		where $\rv_0$ as in \eqref{qr}  which is independent of the $d\times K$ random matrix
		$\matr u\sim_{\text{i.i.d.}}\mathcal N(\matr 0,\mathcal I)$, $\rho_{\rm AT}<1$ is as in Proposition~\ref{pat} and $C$ is an irrelevant constant independent of~$t$. 
	\end{theorem}
	Before proceeding to the proof of Theorem~\ref{Th3}, we present the following high-dimensional analysis of the reconstruction error as a consequence of Theorem~\ref{Th3}:
	\begin{corollary}\label{corerror} Let $\matr \omega ^\star$ be as in \eqref{opt}.
		Suppose the premises of Theorem~\ref{Th3} hold. Then, as $d\to \infty$ we have 
		\begin{equation*}
			\frac 1 d\Vert \matr\omega^\star-\matr \omega_0\Vert_{\texttt{F}}^2-{\rm tr}((\mathcal B_0-\mathcal B^\star)^\top(\mathcal B_0-\mathcal B^\star)+\mathcal C^\star)\overset{a.s.}{\rightarrow}0\;,
		\end{equation*}
		where $\mathcal B^\star$ and $\mathcal C^\star$ are as in \eqref{fixed} and $\mathcal B_0\bydef \langle \rv_0,\matr \omega_0 \rangle$, see~\eqref{qr}.
		\begin{proof}
			See Appendix~\ref{proof_corerror}.
		\end{proof}
	\end{corollary}
	\subsection{Proof of Theorem~\ref{Th3}}
	Since the optimization \eqref{opt} is $\lambda_0$-strongly convex \cite{boyd2004convex} (i.e., the spectral norm of the Hessian of the optimization is bounded below by $\lambda_0$) we can bound the distance between any point $\matr \omega\in \RR^{d\times K}$ and the
	optimal point $\matr\omega^\star$ as \cite[Eq. 9.11]{boyd2004convex}
	\begin{equation}
		\Vert \matr \omega-\matr \omega^\star\Vert_\texttt{F}
		\leq \frac {2} {\lambda_0} \Vert\matr G({\matr \omega})\Vert_\texttt{F}\label{bound1}
	\end{equation}
	where $\Gm$ is the gradient matrix of the optimization \eqref{opt} w.r.t. $\matr \omega$, i.e.,
	\begin{equation}
		\matr G({\matr \omega})=\lambda_0\matr \omega+\matr X^\top l'(\matr X^\top\matr\omega;\matr y)\;.\label{gradient}
	\end{equation}
	Let $\matr \omega\equiv\matr \omega^{(t)}$ for $t>1$ be constructed by the AMP dynamics \eqref{alg2} initialized with $\mathcal B^{(1)}=\mathcal B^\star $ and $\mathcal C^{(1,1)}=\mathcal C^\star $. We then write 
	\begin{align*}
		\matr G(\matr \omega^{(t)})&
		=\lambda_0\matr \omega^{(t)}+\matr X^\top l'(m(\matr \gamma^{(t-1)};\matr y)+(\matr \gamma^{(t)}-\matr \gamma^{(t-1)});\matr y)\nonumber\\
		&\overset{(a)}{=}\lambda_0\matr \omega^{(t)}+\matr X^\top l'(m(\matr \gamma^{(t-1)};\matr y);\matr y)+\matr X^\top \matr h^{(t)}\\
		&\overset{(b)}{=}\lambda_0\matr \omega^{(t)}+\matr X^\top  f(\matr\gamma^{(t)};\matr y)\mathcal V^{\star}+\matr X^\top \matr h^{(t)}\\
		&=\matr \omega^{(t)}(\lambda_0\mathcal I-\mathcal V^{\star})-\alpha\matr \omega^{(t-1)}\mathcal Q^{(t-1)}\mathcal V^\star +\matr X^\top \matr h^{(t)}\\
		&\overset{(c)}{=}(\matr \omega^{(t)}-\matr \omega^{(t-1)})(\lambda_0\mathcal I-\mathcal V^{\star}) +\matr X^\top \matr h^{(t)}\;.
	\end{align*}
	In step (a) we have carried out the mean-value theorem such that the rows of $\matr h^{(t)}$ read in the form 
	\begin{equation}
		( h^{(t)})^{[i]}\bydef ((\gamma^{(t)})^{[i]}-(\gamma^{(t-1})^{[i]})l''(\xi_i;y^{[i]}),~~\forall i\in[n]\;.
	\end{equation} 
	Note that 
	$\Vert \matr h^{(t)}\Vert_{\texttt{F}}\leq L_{l''}\Vert \matr\gamma^{(t)}-\matr \gamma^{(t-1)}\Vert_{\texttt{F}}$ where $L_{l''}$ denotes an upper bound of the spectral norm of the Hessian $l''(\theta;y)$. In step (b) we use \eqref{flrelation} and
	in step (c) we use the fact that given $\mathcal B^{(1)}=\mathcal B^\star$ and $\mathcal C^{(1,1)}=\mathcal C^\star$, we have from \eqref{vnice} that $\alpha \mathcal Q^{(t)}\mathcal V^\star=\lambda_0\mathcal I-\mathcal V^\star$ for all $t$. Thus, we have from \eqref{bound1}
	\begin{align}
		\Vert{\matr \omega^{(t)}}-\matr \omega^\star\Vert_\texttt{F}&\leq \frac 2{\lambda_0}\Vert\lambda_0\mathcal I-\mathcal V^{\star} \Vert_2\Vert\matr \omega^{(t)}-\matr \omega^{(t-1)}\Vert_\texttt{F} \nonumber \\
		&+\frac{L_{l''}}{\lambda_0}\Vert\matr X \Vert_2\Vert\matr \gamma^{(t)}-\matr \gamma^{(t-1)} \Vert_\texttt{F}\label{goodbound}\;. 
	\end{align}
	Here, $\Vert \cdot \Vert_2$ stands for the spectral norm of the matrix in the argument and we have e.g. from \cite[Theorem~2.7]{davidson2001local} that $\Vert \matr X\Vert_2 =\Op{1}$. Furthermore, it is easy to verify that all the premises of Theorem~\ref{Th2} are fulfilled by the premises of Theorem~\ref{Th3}.
	Then, the thesis is evident from  Theorem~\ref{Th2}.

	\section{Simulation Results}\label{Simulations}
	We consider the application of the convex optimization \eqref{opt} with the cross-entropy loss function \eqref{cel}. We generate $\matr \omega_0$ such that $\langle \matr \omega_0,\matr \omega_0\rangle=\mathcal I$ and $\matr y$ according to the log-likelihood $\ln p_0(y\vert\theta)=-l(\theta;y)$. The number of classes is $K=3$. We fix $\alpha=n/d=2$, which becomes critical as $\lambda_0\to 0$. We simulate the AMP dynamics~\eqref{alg2} using the \emph{Householder dice} implementation \cite{Yue21} which allows to simulate the dynamics on a standard personal computer up to $d=10^{6}$ (instead of $10^4$ with a direct implementation).
	We have the rate of convergence 
	\begin{align}
		\lim_{t\to \infty}\lim_{d\to\infty}\frac{\Vert \matr\omega^{(t+1)}-\matr \omega^{(t)}\Vert_{\texttt{F}}^2}{\Vert \matr \omega^{(t)}-\matr \omega^{(t-1)}\Vert_{\texttt{F}}^2}&\overset{a.s.}{=}
		\rho_{\rm AT}\;.\label{at_con}
	\end{align}
	\begin{figure}[t]
		\centering
			\includegraphics[width=1.02\linewidth]{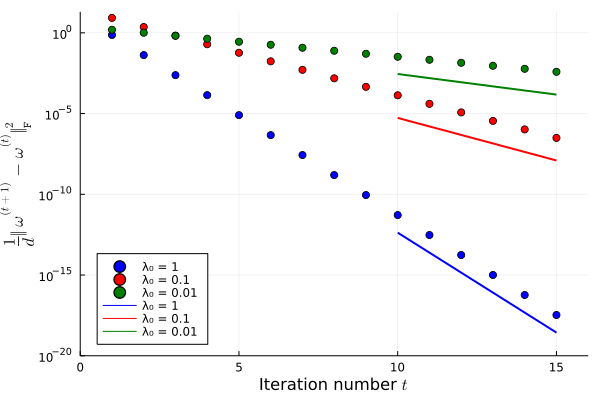}
		\caption{The convergence of the AMP dynamics with $d=10^{5}$ and $\alpha=2$. The straight lines on the interval $10\leq t\leq 15$ represent $\rho_{\rm AT}^{t}$. The experiments are based on single instances (for each $\lambda_0$) of the AMP dynamics.} 
		\label{fig.MMSE}    	
	\end{figure}We illustrate this result in Figure~1.  
	The numerical results suggest that for $\alpha=2$ we have $\rho_{\rm AT}\to 1$ as $\lambda_0\to0$ (and we observe that obtaining the numerical value of $\rho_{\rm AT}$ is more difficult the smaller $\lambda_0$ is and we were unable to obtain a numerical value of $\rho_{\rm AT}$ for $\lambda_0=0$). This is similar to the so-called \emph{Gardner instability}, see \cite[Eq (4)]{gardner1989three}.

	\section{Conclusion}\label{Conclusion}
	We have presented a convergence analysis of the dynamics of an AMP with nonseparable multivariate nonlinearities and its application to multi-class classification. The analysis reveals a necessary (and sufficient) condition for dynamical stability, i.e., $\rho_{\rm AT}<1$ (see \eqref{AT}). We have shown that this condition always holds for ridge-regularized ``softmax'' regression type applications, which are strongly convex problems.
	On the other hand, $\rho_{\rm AT}<1$ could only hold for some ``region'' of model parameter values of the non-convex (or not strictly convex) problems, we expect that the stability criteria could be an important aspect for the analysis of non-convex problems.
	
	It would be interesting to extend the convergence analysis to the generalized AMP setting. This would allow us to analyze, for example, the convex optimization \eqref{opt} with a generalized nonlinear regularization term. Details are discussed elsewhere.
	\section*{Acknowledgment}
	This work was supported by the German Research Foundation, Deutsche Forschungsgemeinschaft (DFG), under Grant ‘RAMABIM’ with No. OP 45/9-1, by the US
	National Science Foundation under Grant CCF-1910410, and by the Harvard FAS Dean’s Competitive Fund for Promising Scholarship.
	\appendices
\section{Concentration Inequalities with \texorpdfstring{$\mathcal L^p$}{TEXT} norm} \label{preliminariesop} 
Here we present some elementary results on concentration inequalities with $\mathcal L^p$ norm. 
\begin{lemma}\label{lemma:op_properties}
Consider the (scalar) random variables ${A} = \Op{\kappa_d}$ and ${B} = \Op{\tilde\kappa_d}$, where $\kappa_d$ and $\tilde \kappa_d$ are two positive sequences indexed by $d$. Then the following properties hold:
	\begin{align}
			{A + B} &= \Op{\max(\kappa_d, \tilde \kappa_d)}\label{eq:op_sum}\\
			{AB}&= \Op{\kappa_d \tilde \kappa_d}\label{eq:op_prod}\\
			\sqrt{A}&=\Op{\sqrt{\kappa_d}}.\label{eq:op_sqrt_pos}
		\end{align}
		Moreover, let $C > 0$ be a constant and ${A}=\Op{\kappa_d}$ with $\vert {A}\vert\leq C$. Then, we have 
		\begin{align}
			\sqrt{C + {A}} -\sqrt{C} &= \Op{\kappa_d}\;.\label{eq:op_sqrt}
		\end{align}
		\begin{proof}
The results \eqref{eq:op_sum} and \eqref{eq:op_prod} follow from the Minkowski inequality and H\"older inequality, (i.e., $\Vert {A}+{B} \Vert_{\mathcal L^p}\leq\Vert {A}\Vert_{\mathcal L^p}+\Vert{B} \Vert_{\mathcal L^p}$ and $\Vert {A}{B} \Vert_{\mathcal L^p}\leq\Vert {A}\Vert_{\mathcal L^{2p}}\Vert{B} \Vert_{\mathcal L^{2p}}$), respectively. The result \eqref{eq:op_sqrt_pos} follows from $ \Vert {A}^{1/2} \Vert_{\mathcal L^{p}}\leq \Vert {A}^{1/2} \Vert_{\mathcal L^{2p}}=\Vert {A} \Vert_{\mathcal L^{p}}^{\frac 1 2}$. Finally, the result \eqref{eq:op_sqrt} is evident:
  \begin{align*}
\vert\sqrt{C + {A}} -\sqrt{C}\vert= \frac{\vert {A}\vert}{\sqrt{C + {A}} +\sqrt{C}}\leq \frac{\vert {A}\vert}{\sqrt{C}}=\Op{\kappa_d}\;.
  \end{align*} 
\end{proof}
\end{lemma}
\begin{lemma}\label{keycon}\cite[Lemma~7.8]{erdHos2017dynamical}
For $m\asymp d$, consider a random vector $\matr a\in \RR^{m}$ where $\av\sim_\text{i.i.d.} A$ and ${A}=\Op{1}$. Then,  
\[\frac 1{m} \sum_{i\in[m]}{a}_i=\mathbb E[{A}]+\Op{d^{-1/2}}\;.\]
	\end{lemma}
	
\begin{lemma}\label{eq:ip_concentration}
		For $m\asymp d$, consider the random vectors $\av,\bv\in \RR^{m}$ where  $ \av\sim_\text{i.i.d.} A$ and  $ \bv\sim_\text{i.i.d.} B$  with ${A}=\Op{1}$ and $B=\Op{1}$. Then, for any $\hat{\av}\simeq\av$ and ${\hat \bv}\simeq \bv$, we have
		\begin{equation}
			\langle \hat{ \av},{\hat \bv}\rangle = \mathbb E[A B] + \Op{d^{-\frac 1 2}}\;.
		\end{equation}
		\begin{proof}
			Let ${\matr \delta_a}\bydef\hat{\av}-{\av} $ and ${\matr \delta_b}\bydef\hat{\bv}-{\bv} $ with noting that e.g. ${\matr \delta_{a}}=\Op{1}$. From Lemma~\ref{keycon} we have $ {\av}=\Op{\sqrt{d}}$ and $ {\bv}=\Op{\sqrt{d}}$. Hence, from the properties \eqref{eq:op_sum} and \eqref{eq:op_prod}  we get
			\begin{equation}
			\langle {\hat \av},{\hat \bv}\rangle-\langle{ \av},{ \bv}\rangle=\langle{ \av},{\matr\delta_b}\rangle+\langle{ \bv},{\matr\delta_a}\rangle+\langle{\matr \delta_a},{\matr\delta_b}\rangle =\Op{d^{-\frac 1 2}}\;.
			\end{equation}
			Also, from \eqref{eq:op_prod} we have
			$AB=\Op{1}$; so that from Lemma~\ref{keycon} we get
			$\langle { \av},{\bv}\rangle= \mathbb E[AB] + \Op{d^{-\frac 1 2}}$. This completes the proof.
		\end{proof}	
	\end{lemma}

\begin{lemma}\label{rem1}
For $m\asymp d$ let the random vectors $\matr g, \matr r\in \R^{m\times K}$ be independent with $\matr g\sim_\text{i.i.d.}\mathcal N(0,\mathcal I)$ and $\langle \matr r,\matr r\rangle=\mathcal I$. Then, $\langle \matr r,\matr g \rangle =\Op{d^{-\frac 1 2}}$.
\begin{proof}
Let $\mathcal G\bydef\sqrt m\langle \rv,\gv \rangle$. Notice that $\mathcal G\sim_{\text{i.i.d.}}\mathcal N(\matr 0,\mathcal I)$. Hence, $\mathcal G=\Op{1}$ which completes the proof. 
\end{proof}
\end{lemma}

\newcommand{\secplemma}{Proof of Lemma~\ref{TContradiction}}
\section{\secplemma}\label{dTContradiction}
For the sake of notational compactness, let
\begin{equation}
f_0(U)\bydef f(G_0\mathcal B^\star+U;Y). 
\end{equation}
It is also useful to write the mapping \eqref{mappin} as
\begin{equation}
\mathcal T(\mathcal X)=\alpha\mathbb E[f_0(\Psi_1)^\top f_0(\Psi_2)] 
\end{equation}
where $\Psi_1$ and $\Psi_2$ are zero-mean Gaussian vectors independent of the field $(Y,G_0$) with $\mathbb E[\Psi_m^\top\Psi_m]=\mathcal C^\star$ for $m=1,2$ and $\mathbb E[\Psi_1^\top\Psi_2]=\mathcal X$. In particular, by using the characteristic-function representation of the Gaussian distribution we write
\begin{align*}
\mathcal T(\mathcal X)=c\int{\rm d}{\rm P}(Y,G_0){\rm d}U_1 {\rm d}U_2 {\rm d}\Psi_1{\rm d}\Psi_2\; f_0(\Psi_1)^\top f_0(\Psi_2)\nonumber
\\ \times {\rm e}^{-{\rm i}(U_1\Psi_1^\top+U_2\Psi_2^\top)-\frac 1 2 U_1\mathcal C^\star U_1^\top-\frac 1 2 U_2\mathcal C^\star U_2^\top}{\rm e}^{-U_1 \mathcal X U_2^\top}
\end{align*}
where ${\rm d}{\rm P}(Y,G_0)\bydef {\rm d}Y{\rm d}{G}_0~p_0(Y\vert G_0\sqrt{\mathcal C_0})\textswab{g}(G_0\vert \matr 0,\mathcal I)$ and $c\bydef \alpha/(2\pi)^{2K}$. Also, one can verify that the differentiation w.r.t. $\mathcal X_{kk'}$ is interchangeable with the integral above, see e.g. \cite[Lemma~2]{palomar2005gradient}. We will perform differentiation with respect to the symmetric matrix $\mathcal X$ and thereby we need to take the symmetry into account. To that end, we define the so-called elimination and duplication matrices. 
\begin{definition}\cite{magnus1980elimination}
Consider a $K\times K$ matrix $\mathcal X=\mathcal X^\top$. Then, $\overline{\mathcal X}_p$ denotes the $\frac 1  2K(K+1)\times 1$ vector obtained from the $K^2\times 1$ vector $\overline{\mathcal X}$ (see Definition~\ref{vectorization}) by eliminating all supra-diagonal elements of $\mathcal X$. E.g., when $K=3$, we have 
\begin{align*}
\overline{\mathcal X}&=(\mathcal X_{11},\mathcal X_{21},\mathcal X_{31},\mathcal X_{12},\mathcal X_{22},\mathcal X_{32},\mathcal X_{13},\mathcal X_{23}\mathcal X_{33})^\top\\
\overline{\mathcal X}_p&=(\mathcal X_{11},\mathcal X_{21},\mathcal X_{31},\mathcal X_{22},\mathcal X_{32},\mathcal X_{33})^\top.
\end{align*}
\end{definition}

For each $K$ there is a unique $\frac 1  2K(K+1)\times K^2$ projection matrix $\mathcal P$ such that 
\begin{equation}
\overline{\mathcal X}_p=\mathcal P\overline{\mathcal X}. 
\end{equation}
Moreover, for each $K$ there exists a unique $K^2\times \frac 1  2K(K+1)$  ``dublication'' matrix $\mathcal D$ such that 
\begin{equation}
\overline{\mathcal X}=\mathcal D\overline{\mathcal X}_p.
\end{equation}
For the explicit definitions of the elimination and duplication matrices we refer to \cite{magnus1980elimination}. In particular, we will solely need the following property 
\begin{equation}
\mathcal D\mathcal P\overline{\mathcal X}=\overline{\mathcal X}, \quad \text{if}~~\mathcal X=\mathcal X^\top.
\end{equation}
Finally, it is useful to note that 
\begin{align}
\frac{\partial{\rm e}^{-U_1 \mathcal X U_2^\top}}{\partial \bar{\mathcal X}}&= \frac{\partial{\rm e}^{(\overline{-U_2^\top U_1})^\top \overline{\mathcal X}}}{\partial \bar{\mathcal X}}=\frac{\partial \overline{\mathcal X}_p}{\partial \overline{\mathcal X}}\frac{\partial{\rm e}^{(\overline{-U_2^\top U_1})^\top \mathcal D \overline{\mathcal X}_p}}{\partial \overline{\mathcal X}_p}\\
&=\mathcal P^\top \mathcal D^\top (\overline{-U_2^\top U_1})\\
&=\mathcal P^\top \mathcal D^\top({\rm i} U_1)^\top\otimes ({\rm i} U_2)^\top.
\end{align}
By using this identity we have for all $k,k'\in[K]$
\begin{align}
{[{\mathcal T}' (\mathcal X)]_{kk'}}&\bydef\frac{\partial [\mathcal T (\mathcal X)]_{kk'}}{\partial \overline {\mathcal X}}\\
&=\alpha \mathcal P^\top \mathcal D^\top \mathbb E[({\rm i} U_1)^\top\otimes ({\rm i} U_2)^\top  f_{0k}(\Psi_1)f_{0k'}(\Psi_2)]\\
&=\alpha \mathcal P^\top \mathcal D^\top \mathbb  E[f'_{0k}(\Psi_1)\otimes f_{0k'}'(\Psi_2)]\label{usethis}
\end{align}
where for short we adopt the notational setups for $k\in[K]$
\begin{equation}
f_{0k}(\Psi)\bydef (f_0(\Psi))_k \quad \text{and}\quad f'_{0k}(\Psi)\bydef \frac{\partial f_{0k}(\Psi)}{\partial \Psi}.
\end{equation}

\subsection{The proof of the bound (a)} 
For any $\matr 0\leq \mathcal Y\leq\mathcal X\leq \mathcal C^\star$, we introduce the trivial interpolation for $q\in[0,1]$
\begin{equation}
\mathcal G(q)\bydef\mathcal T(q\mathcal X+(1-q)\mathcal Y).
\end{equation} 
Then, by the general mean-value theorem of \cite{mcleod1965mean} we write
\begin{align}
\mathcal T(\mathcal X)-\mathcal T(\mathcal Y)&=\mathcal G(1)-\mathcal G(0)\\
&=\sum_{1\leq i\leq K^2}\lambda_i\mathcal G'(q_i)\label{gmvt}
\end{align}
for some $q_i\in (0,1)$ and $\sum_{1\leq i\leq K^2}\lambda_i=1$.  In particular, by the chain rule, we have for all $k,k'\in[K]$
\begin{equation}
[\mathcal G'(q)]_{kk'}
=([{\mathcal T}'(\mathcal S)]_{kk'})^\top\overline{\mathcal X-\mathcal Y}
\end{equation} 
where $\mathcal S=q\mathcal X+(1-q)\mathcal Y$. 
Then, from \eqref{usethis} we write
\begin{align}
[\mathcal G'(q)]_{kk'}&=\alpha \mathbb E[f_{0k}'(G\sqrt{\mathcal S}+ G'\sqrt{\mathcal C^\star-\mathcal S}) \nonumber \\
&~~~~~\otimes f_{0k'}'(G\sqrt{\mathcal S}+ G''\sqrt{\mathcal C^\star-\mathcal S})]^\top \mathcal D \mathcal P \overline{X-Y} \nonumber \\
&=\alpha \mathbb E[f_{0k}'(G\sqrt{\mathcal S}+ G'\sqrt{\mathcal C^\star-\mathcal S}))\nonumber \\
&~~~~~\otimes f_{0k'}'(G\sqrt{\mathcal S}+ G''\sqrt{\mathcal C^\star-\mathcal S})]^\top \overline{X-Y}\nonumber\\
&=\alpha \mathbb E[f_{0k'}'(G\sqrt{\mathcal S}+ G'\sqrt{\mathcal C^\star-\mathcal S})(\mathcal X-\mathcal Y)\nonumber \\&~~~~\times f_{0k}'(G\sqrt{\mathcal S}+ G''\sqrt{\mathcal C^\star-\mathcal S})^\top]
\end{align}
where the latter equality follows from the property of the vectorization operator $\overline{\matr A\matr X\matr B^\top}=(\matr B\otimes\matr A)\overline{\matr X}$. We then write everything in the matrix notation as 
\begin{align}
\mathcal G'(q)&=\alpha\mathbb E[f_0'(G\sqrt{\mathcal S}+ G'\sqrt{\mathcal C^\star-\mathcal S})^\top \Delta \nonumber\\ &~~~~\times f_0'(G\sqrt{\mathcal S}+ G''\sqrt{\mathcal C^\star-\mathcal S})] \label{ga0}
\end{align}
where we have introduced $\Delta\bydef \mathcal X-\mathcal Y\geq \matr 0$. Moreover, by introducing the auxiliary random operator \[\mathcal F(G)\equiv\mathbb E_{G'}[f_0'(G\sqrt{\mathcal S}+ G'\sqrt{\mathcal C^\star-\mathcal S})]\] we can write $\mathcal G'(q)=\mathbb E[\mathcal F(G)^\top \Delta\mathcal F(G)]$ which implies that $\mathcal T$ is monotonic, i.e. 
\begin{equation}
\mathcal G'(q)\geq \matr 0. 
\end{equation}
\subsection{Proof of the bound (b)}
For short we define the mapping for $\matr 0\leq \mathcal X\leq \mathcal C^{\star}$ 
\begin{align}
\widetilde{\mathcal T}_{\Delta}(\mathcal X)\bydef&\alpha\mathbb E[f_0'(G\sqrt{\mathcal X}+ G'\sqrt{\mathcal C^\star-\mathcal X})^\top \Delta \nonumber \\
&\quad \times f_0'(G\sqrt{\mathcal X}+ G''\sqrt{\mathcal C^\star-\mathcal X})]
\end{align}
for a fixed $\Delta\geq \matr 0$. Then,  for any $\matr 0\leq \mathcal Y\leq\mathcal X\leq \mathcal C^\star$ we study the interpolation for $q\in[0,1]$
\begin{equation}
 \tilde{\mathcal G}_\Delta(q)\bydef \widetilde{\mathcal T}_{\Delta}(q\mathcal X+(1-q)\mathcal Y). 
\end{equation}
We next show that $\tilde{\mathcal G}'_\Delta(q)\geq \matr 0$. By following the argument of \eqref{gmvt}, this will then imply the bound (b).  First, for all $k,k'\in[K]$ we write  by the chain rule
\begin{equation}
[\tilde{\mathcal G}_\Delta'(q)]_{kk'}
=([{\widetilde{\mathcal T}_{\Delta}}'(\mathcal S)]_{kk'})^\top\overline{\tilde \Delta}
\end{equation}
where $\tilde \Delta\bydef \mathcal X-\mathcal Y$ and $\mathcal S=q\mathcal X+(1-q)\mathcal Y$ and
\begin{equation}
{[\widetilde {\mathcal T}_\Delta' (\mathcal X)]_{kk'}}\bydef\frac{\partial [\widetilde {\mathcal T}_\Delta  (\mathcal X)]_{kk'}}{\partial \overline {\mathcal X}}\;. 
\end{equation}
We then obtain
\begin{align}
[\tilde{\mathcal G}_\Delta'(q)]_{kk'}
&=\alpha{\rm tr}\left(\mathbb E[f''_{0k}(G\sqrt{{\mathcal S}}+ G'\sqrt{\mathcal C^\star-{\mathcal S}})\Delta\nonumber \right.\\
&\left. \quad \quad \times f''_{0k'}(G\sqrt{{\mathcal S}}+ G''\sqrt{\mathcal C^\star-{\mathcal S}})^\top]\tilde\Delta\right)\\
&=\alpha{\rm tr}\left(\sqrt{\tilde\Delta}\mathbb E[f''_{0,k}(G\sqrt{{\mathcal S}}+ G'\sqrt{\mathcal C^\star-{\mathcal S}})\Delta\nonumber \right.\\
&\left.\quad \quad \times f''_{0k'}(G\sqrt{{\mathcal S}}+ G''\sqrt{\mathcal C^\star-{\mathcal S}})^\top]\sqrt{\tilde\Delta}\right)
\end{align}
where we have defined $f_{0k}''(\Psi)\bydef\frac{\partial f_{0k}(\Psi)}{\partial \Psi^\top \partial\Psi}$. For further notational compactness, we introduce the  $K\times K^2$ Hessian matrix
\begin{equation}
f_0''(\Psi)\bydef \left[\frac{\partial f_{01}(\Psi)}{\partial \Psi^\top \partial\Psi},\frac{\partial f_{02}(\Psi)}{\partial \Psi^\top\partial\Psi},\ldots, \frac{\partial f_{0K}(\Psi}{\partial \Psi^\top \partial\Psi}\right].
\end{equation}
Then, we write everything in the compact matrix notation 
\begin{align*}
\tilde{\mathcal G}_\Delta'(q)=&
\alpha{\rm tr}_K\left((\mathcal I\otimes \sqrt{\tilde\Delta}) \mathbb E\left[f_0''\left(G\sqrt{{\mathcal S}}+ G'\sqrt{\mathcal C^\star-\tilde{\mathcal S}}\right)^\top\nonumber \right.\right.\\
&\left.\left. \quad  \times \Delta f_0''\left(G\sqrt{\tilde{\mathcal S}}+ G''\sqrt{\mathcal C^\star-\tilde{\mathcal S}}\right)\right](\mathcal I\otimes \sqrt{\tilde\Delta})\right)
\end{align*}
where for a matrix $\mathcal X\in \RR^{K^2\times K^2}$ we write
\begin{equation}
{\rm tr}_K\left(
\begin{array}{ccc}
\mathcal X_{11} & \ldots & \mathcal X_{1K}\\
\vdots& \ddots &\vdots \\
\mathcal X_{K1}&\ldots &\mathcal X_{KK}
\end{array}\right)\equiv \left(
\begin{array}{ccc}
{\rm tr}(\mathcal X_{11}) & \ldots & {\rm tr}(\mathcal X_{1K})\\
\vdots& \ddots &\vdots \\
{\rm tr}(\mathcal X_{K1})&\ldots &{\rm tr}(\mathcal X_{KK})
\end{array}\right)\nonumber 
\end{equation}
where $ \mathcal X_{kk'}\in \RR^{K\times K}$ for all $k,k'$. Equivalently, we have
\begin{equation}
{\rm tr}_K(\mathcal X)=\mathbb E_Z[(\mathcal I\otimes Z)\mathcal X(\mathcal I\otimes Z^\top)]
\end{equation}
where $Z\sim\mathcal N(\matr 0;\mathcal I)$ is  an arbitrary $1\times K$ dim. Gaussian random vector. Thus, if $\mathcal  X\geq \matr 0$ then ${\rm tr}_K(\mathcal X)\geq\matr 0$.  We finally introduce the auxiliary $K\times K^2$ random operator 
\[ \mathcal  F'(G)\equiv \mathbb E_{G'}[f_0''(G\sqrt{\mathcal C^\star-{\mathcal S}}+ G'\sqrt{{\mathcal S}})]\]
and then write
\begin{align}
&\mathbb E[f_0''(G\sqrt{\mathcal S}+ G'\sqrt{\mathcal C^\star-\mathcal S})^\top\Delta f_0''(G\sqrt{\mathcal S}+ G''\sqrt{\mathcal C^\star-\mathcal S})]\nonumber \\
&=\mathbb E[\mathcal F'(G)^\top\Delta \mathcal F'(G)].
\end{align}
Since $\Delta\geq \matr 0$ we have $\mathbb E[\mathcal F'(G)^\top\Delta \mathcal F'(G)]\geq \matr 0$. This completes the proof. 
\section{Dynamical stability: Necessity of 
\texorpdfstring{$\rho_{AT}<1$}{Lg}}
\label{proofremat}
Suppose $\overline{\Delta^{(t)}}=o(1)$
where $o(1)$ stands for a $K^2 \times 1$ vector with $\Vert o(1) \Vert\to 0$ as $t\to \infty$. Then, we use the results \eqref{ga0} and obtain the ``linearized'' updates
\begin{equation}
	\overline {\Delta^{(t+1)}}= \alpha\mathbb E[f'(\Gamma^\star;Y)\otimes f'(\Gamma^\star;Y)]^\top] \overline{\Delta^{(t)}}+o(1)\label{usecontradiction}\;
\end{equation}
which implies when $\rho_{\rm AT}\geq 1$, $\overline {\Delta^{(t)}}\neq o(1)$. Hence, by contradiction, the condition $\rho_{\rm AT}<1$ is necessary for $\overline{\Delta^{(t)}}=o(1)$.

\section{Proof of Proposition~\ref{pat}}\label{proof_AT}
We introduce the auxiliary matrix
\begin{align}
\mathcal E&\bydef (\mathcal I-\alpha\mathbb E[\Lambda^\star(\mathcal V^\star+\Lambda^\star)^{-1}\otimes\Lambda^\star(\mathcal V^\star+\Lambda^\star)^{-1}])\mathcal V^\star \otimes \mathcal V^\star\nonumber \\
&=\mathcal V^\star \otimes \mathcal V^\star-\alpha \mathbb E[\Lambda^\star(\mathcal V^\star +\Lambda^\star)^{-1}\mathcal V^\star \otimes \Lambda^\star(\mathcal V^\star +\Lambda^\star )^{-1}\mathcal V^\star]
\end{align}
where $\Lambda^\star\bydef l''(m(\Gamma^\star;Y);Y)$. By the convexity of the loss function, we note that $\Lambda^\star\geq \matr 0$. On the other hand, we  have
\begin{equation}
\mathcal V^\star=\lambda_0\mathcal I+\alpha \mathbb E[\Lambda^\star(\mathcal V^\star +\Lambda^\star)^{-1}\mathcal V^\star].
\end{equation}
Thus, we get
\begin{equation}
\mathcal V^\star \otimes \mathcal V^\star=\lambda_0\mathcal I\otimes \mathcal V^\star+\alpha \mathbb E[\Lambda^\star(\mathcal V^\star +\Lambda^\star )^{-1}\mathcal V^\star]\otimes \mathcal V^\star
\end{equation}
and thereby $\mathcal E$ reads as
\begin{align}
&\lambda_0\mathcal I\otimes \mathcal V^\star+\alpha \mathbb E [\Lambda^\star(\mathcal V^\star +\Lambda )^{-1}\mathcal V^\star\otimes \mathcal V^\star-\Lambda^\star(\mathcal V^\star +\Lambda^\star )^{-1}\mathcal V^\star]\nonumber \\
&=\lambda_0\mathcal I\otimes \mathcal V^\star+\alpha \mathbb E [\Lambda^\star(\mathcal V^\star +\Lambda^\star )^{-1}\mathcal V^\star\otimes \mathcal V^\star(\mathcal V^\star +\Lambda^\star)^{-1}\mathcal V^\star].\nonumber 
\end{align}
Thus $\mathcal E>\matr 0$ which implies that
\begin{equation}
\mathcal E(\mathcal V^\star \otimes \mathcal V^\star)^{-1}=(\mathcal I-\alpha\mathbb E[f'(\Gamma^\star;Y))\otimes f'(\Gamma^\star;Y)])
\end{equation}
has strictly positive eigenvalues, and thereby $\rho_{\rm AT}<1$. 
\section{Proof of Corollary~\ref{corerror} }\label{proof_corerror}
For short let $\matr \psi\bydef \matr u\sqrt{\mathcal C^\star}$. Then, we have
\begin{align}
&\frac{\Vert \matr \omega_0-(\matr r_0\mathcal B^\star+\matr \psi) \Vert_{\texttt{F}}}{\sqrt d}\nonumber\\
&=\sqrt{
{\rm tr}[(\mathcal B_0-\mathcal B^\star)^\top(\mathcal B_0-\mathcal B^\star)+2\langle \matr r_0,\matr \psi\rangle+\langle\matr \psi,\matr \psi\rangle] }\\
&\overset{(a)}{=}\sqrt{
{\rm tr}[(\mathcal B_0-\mathcal B^\star)^\top(\mathcal B_0-\mathcal B^\star)+\langle\matr \psi,\matr \psi\rangle]+\Op{d^{-\frac 1 2}} }\\
&\overset{(b)}{=}\sqrt{
{\rm tr}[(\mathcal B_0-\mathcal B^\star)^\top(\mathcal B_0-\mathcal B^\star)+\mathcal C^\star]+\Op{d^{-\frac 1 2}}}\\
&\overset{(c)}{=}\sqrt{
{\rm tr}[(\mathcal B_0-\mathcal B^\star)^\top(\mathcal B_0-\mathcal B^\star)+\mathcal C^\star]}+\Op{d^{-\frac 1 2}}\;,
\end{align}
where steps (a), (b), and (c) use Lemma~\ref{rem1}, Lemma~\ref{eq:ip_concentration} and Lemma~\ref{lemma:op_properties}, respectively. Then the thesis follows from an appropriate application of the triangular inequality of Frobenious norm to Theorem~\ref{Th3}.

\section{Sketch of the Proof of Proposition~\ref{Th1}}\label{proof_Th1}
 
The proof is based on the idea of the "Householder dice" representation of AMP dynamics introduced in \cite{lu2021householder}, which is a way to represent the AMP dynamics -- which are coupled by a random matrix ($\Xm$, in our case) -- as equivalent random matrix-free dynamics. We have two main steps: In the first step (detailed in section~\ref{step1}), we use the Gram-Schmidt orthogonalization to represent the dynamics of AMP \eqref{alg2} as an equivalent $\Xm$-free dynamics, called the "Householder dice representation". In the second step (detailed in section~\ref{step2}), we use the Cholesky decomposition, along with the properties of the notion of $\mathcal L^p$ concentration given in Appendix~\ref{preliminariesop}, to derive the high-dimensional equivalence of the Householder dice equivalence [and thereby the high-dimensional equivalent of the original AMP dynamics \eqref{alg2}].

\subsection{The Householder Dice Representation}\label{step1}
We begin with the symmetrization trick \cite{javanmard2013state, Berthier20} to pack the original AMP dynamics \eqref{alg2}, which involves a rectangular random coupling matrix $\Xm$, into a compact form of the dynamics involving only a symmetric random coupling matrix $\Am\in \RR^{m\times m}$ with $m\bydef n+d$. To this end, we introduce
the dynamics for the iteration steps $s=1,2,\cdots,S=2T$
\begin{subequations}\label{Gamp_alg}
\begin{align}
\hv^{(s)}&=\Am\mv^{(s)}-\mv^{(s-1)}\mathcal Q_m^{(s)}\\
\mv^{(s+1)}&=\eta_{s}(\hv^{(s)};\yv)
\end{align}
\end{subequations}
with $\mv^{(0)}\bydef\matr 0$. Here, we have defined the random matrix as 
\begin{equation}
\Am\bydef \frac{1}{\sqrt{1+\alpha}}\left( \begin{array}{cc}
 \sqrt{\alpha}\Zm_n &\Xm  \\
     \Xm^\top& \Zm_d
\end{array}\right)
\end{equation}
where $\Zm_n$ and $\Zm_d$ are arbitrary 
 $n\times n$ and $d\times d$ Gaussian Orthogonal-Ensemble (GOE) random matrices, respectively, e.g. $\Zm_n\sim \frac{1}{\sqrt {2n}}(\Gm+\Gm^\top)$ where $\Gm\sim_{\text{i.i.d.}}\mathcal N(\matr 0,\matr I)$. So, by construction $\Am$ is a GOE random matrix. 
 
To unpack the original AMP dynamics \eqref{alg2} from the dynamics \eqref{Gamp_alg} we set $\mv^{(1)}\equiv{\sqrt{1+\alpha}}{\small\left[\begin{array}{c}
     \matr 0  \\
     \matr \omega^{(1)}
\end{array}\right]}$ and 
\begin{equation} 
\eta_s(\hv;\yv)=\sqrt{1+\alpha}\left\{\begin{array}{cc}
    \left[\begin{array}{c}
     f(\hv_n;\yv)  \\
\matr 0  
\end{array} \right] & s=1,3,5,\cdots  \\
  \left[\begin{array}{c}
     \matr 0  \\ 
     \hv_d 
\end{array} \right]   & s=2,4,6,\cdots
\end{array}\right.\label{mt}
\end{equation}
where for convenience we partition $\hv\in \RR^{(n+d)\times K}$ as
\[\hv\equiv\left[\begin{array}{c}
     \hv_n \\ \hv_d
\end{array} \right]\quad \text{with}~~\hv_n\in \RR^{n\times K}\;.\] 
Also, 
let $\mathcal Q_m^{(s)}=\mathcal Q^{(\frac{s}{2})}$ for $s=2,4,6,\cdots$ and $\mathcal Q_m^{(s)}=\mathcal I$ for $s=3,5,7\cdots$. Hence, we have 
\begin{equation}
\hv_n^{(2t-1)}=\matr \gamma^{(t)}\text{  and   }
\hv_d^{(2t)}=\matr\omega^{(t+1)}\;, \quad t\in[T]\;.
\end{equation}

We will adaptively use the following representation of the GOE random matrix which has been reported for the case $K=1$ in \cite{Conditinoing}.
\begin{lemma}\label{G_const}
Let the random matrices $\tilde{\Am}\in \RR^{m\times m}$, $\vv\in \RR^{m\times K}$, $\uv\in \RR^{m\times K}$ and $\mathcal Z\in\RR^{K\times K}$ be mutually independent.  Let $\tilde\Am$ and $\frac{1}{\sqrt K}\mathcal Z$ be both GOE random matrices, $\langle \vv,\vv \rangle=\mathcal I$ and $\uv\sim_\text{i.i.d.} \mathcal N(\matr 0,\mathcal I)$.
Let $\Pm_\vv^\perp\bydef\Id-\frac 1 m\vv\vv^\top$.
Then, 
\begin{equation}
\Am\bydef\frac{1}{m}\left(\Pm_\vv^\perp\uv\vv^\top
+\frac {\vv \mathcal Z\vv^\top} {\sqrt{m}}+
\vv\uv^\top\Pm_\vv^\perp\right)+
   \Pm_\vv^\perp\tilde\Am\Pm_\vv^\perp
\end{equation}
is also a GOE random matrix and independent of $\vv$. 
\begin{proof}
Since the GOE random matrix is rotational invariant,  without loss of generality we can assume $\vv=\sqrt{m}[\ev_1,\ev_2,\cdots,\ev_K]$ where $\matr \ev_k\in \RR^{m}$ denotes the standard basis vector. Then, the proof is evident.
\end{proof}
\end{lemma}
For example, consider a matrix $\mv \in \RR^{m\times K}$ with $\text{span}(\mv) = \text{span}(\vv)$. Then, from Lemma~\ref{G_const} we write
$
\Am\mv = \Pm_\mv^\perp\uv\langle \vv,\mv \rangle +  \frac {\vv\mathcal Z\langle \vv,\mv \rangle} {\sqrt{m}}\;$
which involves only the lower-dimensional random elements $\uv$ and $\mathcal Z$. To extend this idea to the dynamics $\Am\mv^{(s)}$, 
we first introduce the (block) Gram-Schmidt orthogonalization notation:
Let $\vv^{(0:s)}=\{\vv^{(0)},\vv^{(1)},\ldots ,\vv^{(s-1)}\}$ be a collection  of matrices in $\RR^{m\times K}$ with $\langle \vv^{(i)}, \vv^{(j)}\rangle=\mathcal \delta_{ij}\mathcal I$ for all $i,j$. 
Then, for any $\bv\in \RR^{m\times K}$, by the Gram-Schmidt orthogonalization process
we can always construct the new orthogonal matrix
\[ \vv^{(s)}\bydef \GS{\bv}{\vv^{(0:s-1)}} 
\]
such that $\langle \vv^{(s)}, \vv^{(j)}\rangle=\delta_{sj}\mathcal I$ and
$
\bv=\sum_{0\leq i\leq s}\vv^{(s)}\langle\vv^{(s)},\bv \rangle\;$.
We iteratively employ the Gram-Schmidt process to construct a set of orthogonal matrices $\vv^{(0:s)} \equiv \{\vv^{(0)}, \vv^{(1)}, \ldots, \vv^{(s)}\}$ such that $\text{span}(\vv^{(0: s)}) = \text{span}(\mv^{(1: s)})$ and apply Lemma~\ref{G_const} to obtain the following $\Am$-free equivalent of the dynamics:

\begin{lemma}[The Householder Dice Representation]
Let $\Am$ be the GOE random matrix. Then, the joint probability distribution of the sequence of matrices $\{\hv^{(1:S)},\mv^{(1:S+1)}\}$ generated by dynamics \eqref{Gamp_alg} is equal to that of the same sequence generated by the following dynamics for $s=1,2,\cdots, S$
\begin{subequations}\label{Gamp_mf}
\begin{align}
 \vv^{(s)}&=\mathcal{GS}(\mv^{(s)}\vert \vv^{(0:s-1)}).\\
\hat{\uv}^{(s)}&=\Pm_{\mv^{(0:s)}}^\perp\uv^{(s)}\\
\matr \epsilon^{(s)}&=\frac{1}{\sqrt m}\vv^{(s)}\mathcal Z^{(s)}\\
\hv^{(s)}&= \sum_{0\leq s'\leq s}\left(\hat{\uv}^{(s')}+\matr \epsilon^{(s')}\right)\langle\vv^{(s')},\mv^{(s)}\rangle\nonumber \\
&+\sum_{0\leq s'<s}\vv^{(s)}\left(\langle \hat\uv^{(s')},\mv^{(s)}\rangle-\langle \vv^{(s')},\mv^{(s-1)}\rangle\mathcal Q_m^{(s)}\right)\\
\mv^{(s+1)}&=\eta_{s}(\hv^{(s)},\yv)\;.
\end{align}
\end{subequations}
Here, for each $s\geq 1$ we generate arbitrary Gaussian random matrices $\uv^{(s)}\sim\uv$ and arbitrary symmetric Gaussian random matrices $\mathcal Z^{(s)}\sim\mathcal Z$ where $\uv$ and $\mathcal Z$ as in Lemma~\ref{G_const}. Furthermore, we recall  $\rv_0=\mathcal{GS}(\matr\omega_0)$ and we introduce
\begin{equation}
\mv^{(0)}\bydef \left[\begin{array}{c}
     \matr 0  \\
     \matr \omega_0
\end{array}\right]~\text{and}~ \vv^{(0)}\bydef {\sqrt{1+\alpha}}\left[\begin{array}{c}
     \matr 0  \\
     \rv_0
\end{array}\right]\;.
\end{equation}
Moreover, we introduce the $m\times K$ Gaussian random element 
$
\uv^{(0)}\bydef{\small \left[\begin{array}{c}
     \gv_0  \\
     \uv_0
\end{array}\right]}\in \RR^{m\times K}$
where we recall that $\gv_0=\Xm\rv_0$ and $\uv_0 \sim_{\text{i.i.d.}}\mathcal N(\matr 0;\mathcal I)$ is an arbitrary random element.
Finally, $\Pm_{\mv^{(0:s)}}^\perp\bydef\Id-\frac 1 m \sum_{0\leq s'\leq s} \vv^{(s)}(\vv^{(s)})^\top\;$ which is the projection matrix onto the orthogonal complement of $\text {span}(\mv^{(0:t)})$.
\end{lemma}

\subsection{The High-Dimensional Equivalent}\label{step2}
As a first step, we begin with the following high-dimensional equivalence of the field $\hv^{(s)}$ in \eqref{Gamp_mf}:
\begin{align}
\hv^{(s)}&\simeq \sum_{0\leq s'\leq s}{\uv}^{(s')}\langle\vv^{(s')},\mv^{(s)}\rangle\nonumber \\&+
\sum_{0\leq s'<s}\vv^{(s)}\left(\langle \uv^{(s')},\mv^{(s)}\rangle-\langle \vv^{(s')},\mv^{(s-1)}\rangle\mathcal Q_m^{(s)}\right)\;, \label{heffective}
\end{align}
where we have invoked the results for all $0\leq s'\leq s\leq S$
\begin{align}
\langle\vv^{(s')},\mv^{(s)}\rangle&\overset{(a)}{=}\Op{1}\\
\hat{\uv}^{(s)}-\uv^{(s)}&\overset{(b)}{=}\Op{1}\\
\matr \epsilon^{(s)}&\overset{(c)}{=} \Op{1}\;
\end{align}
Here, to verify $(a)$ it is enough to verify that $\mv^{(s)}=\Op{\sqrt{d}}$ which can be verified inductively over the iteration steps. The step (b) follows from Lemma~\ref{rem1}. The step (c) follows from the fact that $\mathcal Z=\Op{1}$.

We now recall \eqref{mt}, i.e. for all $1\leq s\leq S$ we have
\begin{equation}
\mv^{(s)}=\sqrt{1+\alpha}\left\{\begin{array}{cc}
   \left[\begin{array}{c}
     \matr 0  \\ 
     \matr \omega^{(\frac{s+1}{2})} 
\end{array} \right]  & s=1,3,5,\cdots   \\\\
  \left[\begin{array}{c}
     \fv^{(\frac s 2)}  \\ 
     \matr 0
\end{array} \right]   &  s=2,4,6,\cdots
\end{array}\right.\;.
\end{equation}
We then specialize the Gram-Schmidt process $\vv^{(s)}=\mathcal{GS}(\mv^{(s)}\vert \vv^{(0:s-1)})$. To start we introduce the decompositions 
\begin{equation}
\vv^{(s)}\equiv  \sqrt{1+\alpha}\left[\begin{array}{c}
    \frac {1}{\sqrt \alpha}\tilde\vv_n^{(s)} \\
   ~~ \tilde\vv_d^{(s)} 
\end{array} \right]\;.\quad s\in[S]
\end{equation}
Suppose that we have constructed $\vv^{(s)}$ by
\begin{align*}
\tilde \vv^{(s)}_n&=\left\{\begin{array}{cc}   
  ~~\mathcal {GS}(\fv^{(\frac s 2)}\vert \tilde\vv_n^{(2)},\tilde\vv_n^{(4)},\cdots \tilde\vv_n^{(s-2)})   & ~~~~~~s=2,4,6,\cdots \\
  ~~~~  \matr 0 & ~~~~~~s=1,3,5,\cdots 
\end{array}\right.\\
\tilde \vv^{(s)}_d&=\left\{\begin{array}{cc}
    \matr 0 & s=2,4,6,\cdots \\
      \mathcal {GS}(\matr\omega^{(\frac{s+1}{2})}\vert \rv_0,\tilde\vv_d^{(1)},\tilde\vv_d^{(3)}\cdots,\tilde\vv_d^{(s-2)})   & s=1,3,5,\cdots 
\end{array}\right.
\end{align*}
Indeed, by construction, we have $0\leq s,s'\leq S$
\begin{align}
\langle \vv^{(s)},\vv^{(s')}\rangle&=\delta_{ss'}\mathcal I\\
\text{span}(\mv^{(0:s)})&=\text{span}(\vv^{(0:s)}).
\end{align}
Hence, the only necessary basis elements are $\rv^{(t)}\bydef\tilde\vv_d^{(2t-1)}$  and $\lv^{(t)}\bydef\tilde\vv_n^{(2t)}$ for $t\in[T]$, i.e.,
\begin{align}
\rv^{(t)}&= \mathcal {GS}(\matr\omega^{(t)}\vert \rv_0,\rv^{(1)},\rv^{(2)},\cdots,\rv^{(t-1)})\\
\lv^{(t)}&= \mathcal {GS}(\fv^{(t)}\vert \lv^{(1)},\lv^{(2)},\cdots,\lv^{(t-1)})
\end{align}
and the only necessary (arbitrary) Gaussian elements are 
\begin{align}
\zv^{(t)}&\bydef \uv_d^{(2t-1)} \\
\gv^{(t)}&\bydef  \uv_n^{(2t)}\;.
\end{align}
Recall that $\hv^{(2t)}_n\equiv\matr \gamma^{(t)}$ and $\hv^{(2t+1)}_d\equiv \matr \omega^{(t+1)}$, we then write from \eqref{heffective} 
\begin{align}
 \matr \gamma^{(t)}&\simeq \gv_0\hat{\mathcal B}^{(t)}+ \sum_{1\leq s\leq t}{\gv}^{(s)}\hat{\mathcal B}^{(t,s)}\nonumber\\
&+\underbrace{\sum_{1\leq s<t}\lv^{(s)}\left(\langle \zv^{(s)},\matr\omega^{(t)}\rangle-\hat{\mathcal B}_f^{(t-1,s)}\right)}_{\bydef\tilde{\matr \epsilon}^{(t)} }\label{gammaeff}
 \\
\matr \omega^{(t+1)}&\simeq \rv_0(\alpha\langle \gv_0,\fv^{(t)}\rangle-\alpha\hat{\mathcal B}^{(t)}\mathcal Q^{(t)})
+\sqrt{\alpha}\sum_{1\leq s\leq t}{\zv}^{(s)}\hat{\mathcal B}_f^{(t,s)}\nonumber\\
&+\underbrace{\alpha\sum_{1\leq s<t}\rv^{(s)}\left(\langle \gv^{(s)},\fv^{(t)}\rangle-\hat{\mathcal B}^{(t,s)}
\mathcal Q^{(t)}\right)}_{\bydef\matr \epsilon^{(t)}}\; \label{omegaeff}
\end{align}
where we have defined the $K\times K$ matrices for $1\leq s\leq t\leq T$ 
\begin{align}
\hat{\mathcal B}^{(t)}&\bydef \langle \rv_0,\matr\omega^{(t)}\rangle\\
\hat{\mathcal B}^{(t,s)}&\bydef \langle\rv^{(s)},\matr\omega^{(t)}\rangle\\
\hat{\mathcal B}_f^{(t,s)}&\bydef \langle\lv^{(s)},\fv^{(t)}\rangle\;.
\end{align}
The terms $\tilde{\matr \epsilon}^{(t)}$ and $\matr \epsilon^{(t)}$ stand for the \emph{memory-cancellations} in which we will outline that they concentrate around zero. To express the matrices $\hat{\mathcal B}^{(t,s)}$ and $\hat{\mathcal B}_f^{(t,s)}$
we write for $t\in[T]$
\begin{align}
\hat{\matr \psi}^{(t)}&\bydef \matr\omega^{(t)}-\rv_0\hat{\mathcal B}^{(t)}=\sum_{1\leq s\le t}\rv^{(s)}\hat{\mathcal B}^{(t,s)}\label{psiBhat}\\
\matr f^{(t)}&=\sum_{1\leq s\le t}\lv^{(s)}\hat{\mathcal B}_f^{(t,s)}\;.\label{fBhat}
\end{align}
Recall that e.g. $\langle\rv^{(t)},\rv^{(s)}\rangle=\delta_{ts}\mathcal I$. Thus, from  \eqref{psiBhat} (and resp. \eqref{fBhat}) that  $\hat{\mathcal B}^{(t,s)}$ (and resp. $\hat{\mathcal B}^{(t,s)}$) satisfy the equations of block \emph{Cholesky decomposition} for $1 \leq s\leq t \leq T$:
\begin{align}
\hat{\mathcal B}^{(t,s)}(\hat {\mathcal B}^{(s,s)})^\top&= \langle\hat{\matr \psi}^{(t)},\hat{\matr \psi}^{(s)}\rangle-\sum_{1\leq s' <s}\hat{\mathcal B}^{(t,s')}(\hat{\mathcal B}^{(s,s')})^\top\label{Bhats1}\\
\hat{\mathcal B}_f^{(t,s)}(\hat {\mathcal B}_f^{(s,s)})^\top&= \langle\fv^{(t)},\fv^{(s)}\rangle-\sum_{1\leq s' <s}\hat{\mathcal B}_f^{(t,s')}(\hat{\mathcal B}_f^{(s,s')})^\top\;.\label{Bhats2}
\end{align}

Let $\mathcal H_{t'}$ denote the Hypothesis that for $1\leq s\leq t\leq t'$
\begin{align}
\hat{\mathcal B}^{(t,s)}&={\mathcal B}^{(t,s)}+\Op{d^{-\frac 1 2}} \\
\hat{\mathcal B}_f^{(t,s)}&=\frac{1}{\sqrt \alpha}{\mathcal B}^{(t+1,s+1)}+\Op{d^{-\frac 1 2}}\\
\hat{\mathcal B}^{(t+1)}&={\mathcal B}^{(t+1)}+\Op{d^{-\frac 1 2}}\\
\matr \gamma^{(t)}&\simeq \gv_0{\mathcal B}^{(t)}+ \sum_{1\leq s\leq t}{\gv}^{(s)}{\mathcal B}^{(t,s)}\label{Gamma1}
\\
\matr{\hat \psi}^{(t+1)}&\simeq 
\sum_{1\leq s\leq t}{\zv}^{(s)}{\mathcal B}^{(t+1,s+1)}\;. \label{Omega1}
\end{align} 
Here, $\mathcal B^{(t)}$ is as in Definition~\ref{SE} and the blocks $\mathcal B^{(t,s)}$  satisfy the equations
of block Cholesky decomposition
\begin{subequations}\label{BCDeq}
\noeqref{BCDeq1,BCDeq2}
\begin{align}
{\mathcal B}^{(s,s)}&=\text{chol}\left(\mathcal C^{(s,s)}- \sum_{1\leq s'<s}{\mathcal B}^{(s,s')}({\mathcal B}^{({s,s'})})^\top\right)\label{BCDeq1}\\
{\mathcal B}^{(t,s)}({\mathcal B}^{(s,s)})^\top&=\mathcal C^{(t,s)}- \sum_{1\leq s'<s}{\mathcal B}^{(t,s')}({\mathcal B}^{({s,s'})})^\top\;.\label{BCDeq2}
\end{align}
\end{subequations}
for each $1\leq s\leq t\leq T+1$ where $\mathcal B=\text{chol}(\mathcal A)$ for $\mathcal A\geq 0$ is a lower-triangular matrix such that $\mathcal A=\mathcal B \mathcal B^\top$. 

Using the perturbation idea of \cite[Section 5.4]{Berthier20}  one can verify that in proving Theorem~\ref{Th1} we can assume without loss of generality that
\begin{equation}
\mathcal C^{(1:T+1)}>\matr 0.\label{pda}
\end{equation}
Here,
$\mathcal C^{(1:T+1)}$ denote the $(T+1)K\times (T+1)K$ matrix  with its the $(t,s)$ indexed $K\times K$ block matrix is $\mathcal C^{(t,s)}$,i.e.,
\[ 
\mathcal C^{(t,s)}=(e_t^\top\otimes \mathcal I)\mathcal C^{(1:T+1)}(e_s\otimes \mathcal I)\quad \forall t,s\in [t']\;. \label{notation_set_up}  \]
where $e_{t}$ is the $1\times (T+1)$ dimensional standard basis vector, i.e., $(e_{t})_{s}=\delta_{ts}$. The condition~\eqref{pda} implies the diagonal blocks ${\mathcal B}^{(s,s)}$ for $s\in[T+1]$ are all non-singular. Hence, the blocks ${\mathcal B}^{(t,s)}$  for each $1\leq s\leq t\leq T+1$ can be uniquely constructed through \eqref{BCDeq}. 

Then, using Lemma~\ref{lemma_general_con} below along with the properties of the notion of $\mathcal L^p$ concentration in Appendix~\ref{preliminariesop} one can verify that $H_1$ holds and $H_{t'-1}$ implies $H_{t'}$ for any $t'>1$. This implies Proposition~\ref{Th1}.  
\begin{lemma}\label{lemma_general_con}
 Let $f(\gamma;y)$ be differentiable and Lipschitz continuous w.r.t $\gamma$ and $f(0;Y)=\Op{1}$ where $Y$ as in \eqref{YG0}.
 Suppose \eqref{Gamma1} holds for each $t\in [t']$. Then, for all $1\leq s\leq t\leq t'$ 
\begin{subequations}
\label{concen2}
\noeqref{res3}
\begin{align}
\langle f(\matr \gamma^{(t)};\matr y),f(\matr \gamma^{(s)};\matr y) \rangle&=\frac 1 \alpha \mathcal C^{(t+1,s+1)}+\Op{d^{-\frac 1 2 }}\label{res3}\\
\langle {\matr g}_0, f(\matr \gamma^{(t)};\matr y)\rangle&=\mathbb E[G_0^\top f(\Gamma^{(t)};Y)]+\Op{d^{-\frac 1 2 }}\label{res4}\\
\langle {\matr g}^{(s)},f(\matr \gamma^{(t)};\matr y)\rangle&=\mathcal B^{(t,s)}\mathcal Q^{(t)}
+\Op{d^{-\frac 1 2 }}\label{res5}\;.
\end{align}	
\end{subequations}
In \eqref{res4} the random vectors $\{G_0,\Gamma^{(t)}\}$ are as in Definition~\ref{SE}. 
\begin{proof}
By the Lipschitz property of $f$ we have: $f(\matr \gamma^{(t)};\matr y)\simeq \tilde{\matr f}^{(t)}$ where $  \tilde{\matr f }^{(t)}\sim_{\text{i.i.d.}}f(\Gamma^{(t)};Y)$ for all $t\in[t']$  with
\[ \Gamma^{(t)}=G_0\mathcal B^{(t)}+\underbrace{\sum_{1\leq s\leq t}G^{(s)}{\mathcal B}^{(t,s)}}_{\sim \Psi^{(t)}}\;,\]
where $\Psi^{(t)}$ as in Definition~\ref{SE}; 
and the condition $f(0;Y)=\Op{1}$ implies  $f(\Gamma^{(t)};Y)=\Op{1}$. Third, as to the latter result \eqref{res5} we note that 
from Stein's lemma that 
$ \mathbb E [(G^{(s)})^\top f(\Gamma^{(t)};Y)]=\mathcal B^{(t,s)}\mathcal Q^{(t)}\;.$ Then, the results \eqref{concen2} follow  with the appropriate applications of Lemma~\ref{eq:ip_concentration}.
\end{proof}
\end{lemma}
 
	\bibliographystyle{IEEEtran}
	\bibliography{mybib}

\begin{thebibliography}{10}
\providecommand{\url}[1]{#1}
\csname url@samestyle\endcsname
\providecommand{\newblock}{\relax}
\providecommand{\bibinfo}[2]{#2}
\providecommand{\BIBentrySTDinterwordspacing}{\spaceskip=0pt\relax}
\providecommand{\BIBentryALTinterwordstretchfactor}{4}
\providecommand{\BIBentryALTinterwordspacing}{\spaceskip=\fontdimen2\font plus
\BIBentryALTinterwordstretchfactor\fontdimen3\font minus
  \fontdimen4\font\relax}
\providecommand{\BIBforeignlanguage}[2]{{%
\expandafter\ifx\csname l@#1\endcsname\relax
\typeout{** WARNING: IEEEtran.bst: No hyphenation pattern has been}%
\typeout{** loaded for the language `#1'. Using the pattern for}%
\typeout{** the default language instead.}%
\else
\language=\csname l@#1\endcsname
\fi
#2}}
\providecommand{\BIBdecl}{\relax}
\BIBdecl

\bibitem{Kabashima2009typical}
Y.~Kabashima, T.~Wadayama, and T.~Tanaka, ``A typical reconstruction limit for
  compressed sensing based on lp-norm minimization,'' \emph{Journal of
  Statistical Mechanics: Theory and Experiment}, vol. 2009, no.~09, p. L09003,
  2009.

\bibitem{Bayati11}
M.~Bayati and A.~Montanari, ``The lasso risk for gaussian matrices,''
  \emph{IEEE Transactions on Information Theory}, vol.~58, no.~4, pp.
  1997--2017, 2012.

\bibitem{vehkapera2016analysis}
M.~Vehkaper{\"a}, Y.~Kabashima, and S.~Chatterjee, ``Analysis of regularized ls
  reconstruction and random matrix ensembles in compressed sensing,''
  \emph{IEEE Transactions on Information Theory}, vol.~62, no.~4, pp.
  2100--2124, 2016.

\bibitem{Christos18}
C.~Thrampoulidis, E.~Abbasi, and B.~Hassibi, ``Precise error analysis of
  regularized $m$ -estimators in high dimensions,'' \emph{IEEE Transactions on
  Information Theory}, vol.~64, no.~8, pp. 5592--5628, 2018.

\bibitem{Cedric23}
C.~Gerbelot, A.~Abbara, and F.~Krzakala, ``Asymptotic errors for
  teacher-student convex generalized linear models (or: How to prove
  kabashima’s replica formula),'' \emph{IEEE Transactions on Information
  Theory}, vol.~69, no.~3, pp. 1824--1852, 2023.

\bibitem{thrampoulidis2020theoretical}
C.~Thrampoulidis, S.~Oymak, and M.~Soltanolkotabi, ``Theoretical insights into
  multiclass classification: A high-dimensional asymptotic view,''
  \emph{Advances in Neural Information Processing Systems}, vol.~33, pp.
  8907--8920, 2020.

\bibitem{thrampoulidis2015regularized}
C.~Thrampoulidis, S.~Oymak, and B.~Hassibi, ``Regularized linear regression: A
  precise analysis of the estimation error,'' in \emph{Conference on Learning
  Theory}.\hskip 1em plus 0.5em minus 0.4em\relax PMLR, 2015, pp. 1683--1709.

\bibitem{loureiro2021learning}
B.~Loureiro, G.~Sicuro, C.~Gerbelot, A.~Pacco, F.~Krzakala, and
  L.~Zdeborov{\'a}, ``Learning gaussian mixtures with generalized linear
  models: Precise asymptotics in high-dimensions,'' \emph{Advances in Neural
  Information Processing Systems}, vol.~34, pp. 10\,144--10\,157, 2021.

\bibitem{celentano2022fundamental}
M.~Celentano and A.~Montanari, ``Fundamental barriers to high-dimensional
  regression with convex penalties,'' \emph{The Annals of Statistics}, vol.~50,
  no.~1, pp. 170--196, 2022.

\bibitem{cornacchia2022learning}
E.~Cornacchia, F.~Mignacco, R.~Veiga, C.~Gerbelot, B.~Loureiro, and
  L.~Zdeborov{\'a}, ``Learning curves for the multi-class teacher-student
  perceptron,'' \emph{arXiv preprint arXiv:2203.12094}, 2022.

\bibitem{Rangan_2012}
S.~Rangan, A.~K. Fletcher, and V.~K. Goyal, ``Asymptotic analysis of {MAP}
  estimation via the replica method and applications to compressed sensing,''
  \emph{{IEEE} Transactions on Information Theory}, vol.~58, no.~3, pp.
  1902--1923, mar 2012.

\bibitem{AT}
J.~R. L.~D. Almeida and D.~J. Thouless, ``{Stability of the
  Sherrington-Kirkpatrick solution of a spin glass model},'' \emph{Journal of
  Physics A: Mathematical and General}, vol.~11, no.~5, p. 983, 1978.

\bibitem{Bolthausen}
E.~Bolthausen, ``{An iterative construction of solutions of the TAP equations
  for the Sherrington–Kirkpatrick model},'' \emph{Communications in
  Mathematical Physics}, vol. 325, no.~1, pp. 333--366., October 2014.

\bibitem{Opper16}
M.~Opper, B.~{\c{C}}akmak, and O.~Winther, ``{A theory of solving TAP equations
  for Ising models with general invariant random matrices},'' \emph{Journal of
  Physics A: Mathematical and Theoretical}, vol.~49, no.~11, p. 114002, 2016.

\bibitem{ccakmak2022analysis}
\BIBentryALTinterwordspacing
B.~{\c{C}}akmak, Y.~M. Lu, and M.~Opper, ``Analysis of random sequential
  message passing algorithms for approximate inference,'' \emph{Journal of
  Statistical Mechanics: Theory and Experiment}, vol. 2022, no.~7, p. 073401,
  jul 2022. [Online]. Available:
  \url{https://dx.doi.org/10.1088/1742-5468/ac764a}
\BIBentrySTDinterwordspacing

\bibitem{takahashi2022macroscopic}
T.~Takahashi and Y.~Kabashima, ``Macroscopic analysis of vector approximate
  message passing in a model-mismatched setting,'' \emph{IEEE Transactions on
  Information Theory}, 2022.

\bibitem{javanmard2013state}
A.~Javanmard and A.~Montanari, ``State evolution for general approximate
  message passing algorithms, with applications to spatial coupling,''
  \emph{Information and Inference: A Journal of the IMA}, vol.~2, no.~2, pp.
  115--144, 2013.

\bibitem{Berthier20}
R.~Berthier, A.~Montanari, and P.-M. Nguyen, ``{State evolution for approximate
  message passing with non-separable functions},'' \emph{Information and
  Inference: A Journal of the IMA}, vol.~9, no.~1, pp. 33--79, 01 2019.

\bibitem{zhong2021approximate}
X.~Zhong, T.~Wang, and Z.~Fan, ``Approximate message passing for orthogonally
  invariant ensembles: Multivariate non-linearities and spectral
  initialization,'' \emph{arXiv preprint arXiv:2110.02318}, 2021.

\bibitem{gerbelot2021graph}
C.~Gerbelot and R.~Berthier, ``Graph-based approximate message passing
  iterations,'' \emph{arXiv preprint arXiv:2109.11905}, 2021.

\bibitem{fengler}
A.~Fengler, S.~Haghighatshoar, P.~Jung, and G.~Caire, ``Non-bayesian activity
  detection, large-scale fading coefficient estimation, and unsourced random
  access with a massive mimo receiver,'' \emph{IEEE Transactions on Information
  Theory}, vol.~67, no.~5, pp. 2925--2951, 2021.

\bibitem{ccakmak2024inference}
B.~\c{C}akmak, E.~Gkiouzepi, M.~Opper, and G.~Caire, ``Joint message detection
  and channel estimation for unsourced random access in cell-free user-centric
  wireless networks,'' \emph{arXiv preprint arXiv:2304.12290}, 2024.

\bibitem{horn2012matrix}
R.~A. Horn and C.~R. Johnson, \emph{Matrix analysis}.\hskip 1em plus 0.5em
  minus 0.4em\relax Cambridge university press, 2012.

\bibitem{boyd2004convex}
S.~P. Boyd and L.~Vandenberghe, \emph{Convex optimization}.\hskip 1em plus
  0.5em minus 0.4em\relax Cambridge university press, 2004.

\bibitem{davidson2001local}
K.~R. Davidson and S.~J. Szarek, ``Local operator theory, random matrices and
  banach spaces,'' \emph{Handbook of the geometry of Banach spaces}, vol.~1,
  no. 317-366, p. 131, 2001.

\bibitem{Yue21}
Y.~M. Lu, ``Householder dice: A matrix-free algorithm for simulating dynamics
  on gaussian and random orthogonal ensembles,'' \emph{IEEE Transactions on
  Information Theory}, vol.~67, no.~12, pp. 8264--8272, 2021.

\bibitem{gardner1989three}
E.~Gardner and B.~Derrida, ``Three unfinished works on the optimal storage
  capacity of networks,'' \emph{Journal of Physics A: Mathematical and
  General}, vol.~22, no.~12, p. 1983, 1989.

\bibitem{erdHos2017dynamical}
L.~Erd{\H{o}}s and H.-T. Yau, \emph{A dynamical approach to random matrix
  theory}.\hskip 1em plus 0.5em minus 0.4em\relax American Mathematical Soc.,
  2017, vol.~28.

\bibitem{palomar2005gradient}
D.~P. Palomar and S.~Verd{\'u}, ``Gradient of mutual information in linear
  vector gaussian channels,'' \emph{IEEE Transactions on Information Theory},
  vol.~52, no.~1, pp. 141--154, 2005.

\bibitem{magnus1980elimination}
J.~R. Magnus and H.~Neudecker, ``The elimination matrix: some lemmas and
  applications,'' \emph{SIAM Journal on Algebraic Discrete Methods}, vol.~1,
  no.~4, pp. 422--449, 1980.

\bibitem{mcleod1965mean}
R.~M. McLeod, ``Mean value theorems for vector valued functions,''
  \emph{Proceedings of the Edinburgh Mathematical Society}, vol.~14, no.~3, pp.
  197--209, 1965.

\bibitem{lu2021householder}
Y.~M. Lu, ``Householder dice: A matrix-free algorithm for simulating dynamics
  on gaussian and random orthogonal ensembles,'' \emph{IEEE Transactions on
  Information Theory}, vol.~67, no.~12, pp. 8264--8272, 2021.

\bibitem{Conditinoing}
\BIBentryALTinterwordspacing
J.~Qiu. {Conditioning and Bolthausen’s Lemma}. [Online]. Available:
  \url{https://www.jiazeqiu.com/uploads/1/3/6/1/136158820/stat\_217\_section\_6.pdf}
\BIBentrySTDinterwordspacing

\end{thebibliography}
\end{document}